\documentclass{article}

\usepackage[nonatbib, final]{nips_2017}
\newcommand{\citep}{\cite}
\newcommand{\citet}{\cite}

\usepackage{graphicx} 
\usepackage{subfigure} 

\usepackage[utf8]{inputenc} 
\usepackage[T1]{fontenc}    
\usepackage{hyperref}       
\usepackage{url}            
\usepackage{booktabs}       
\usepackage{amsfonts}       
\usepackage{nicefrac}       
\usepackage{microtype}      

\usepackage{bm, bbm}
\usepackage{inputenc}
\usepackage{amsmath,amssymb,amsthm}
\usepackage{amsthm}
\usepackage{algorithm, algorithmic}
\usepackage[makeroom]{cancel}

\newtheorem{example}{Example}

\newtheorem{lemma}{Lemma}

\DeclareMathOperator*{\argmin}{arg\,min}

\newcommand{\ica}{\hspace{0.25cm}}

\usepackage{capt-of}

\title{Approximate Inference with Amortised MCMC}

\author{
  Yingzhen Li \\
  University of Cambridge\\
  Cambridge, CB2 1PZ, UK \\
  \texttt{yl494@cam.ac.uk} \\
  \And
  Richard E.~Turner \\
  University of Cambridge\\
  Cambridge, CB2 1PZ, UK \\
\texttt{ret26@cam.ac.uk} \\
  \And
  Qiang Liu \\
  Dartmouth College\\
  Hanover, NH 03755, USA \\
\texttt{qiang.liu@dartmouth.edu} \\
}

\newcommand{\vparam}{\bm{\phi}}	

\newcommand{\z}{\bm{z}}

\begin{document}

\maketitle

\begin{abstract} 
We propose a novel approximate inference framework that approximates a target distribution by amortising the dynamics of a user-selected Markov chain Monte Carlo (MCMC) sampler. 
The idea is to initialise MCMC using samples from an approximation network, apply the MCMC operator to improve these samples, and finally use the samples to update the approximation network thereby improving its quality.
This provides a new generic framework for approximate inference, allowing us to deploy highly complex, or implicitly defined approximation families with intractable densities,  
including approximations produced by warping a source of randomness through a deep neural network. Experiments consider Bayesian neural network classification and image modelling with deep generative models. Deep models trained using amortised MCMC are shown to generate realistic looking samples as well as producing diverse imputations for images with regions of missing pixels.
\end{abstract}

\section{Introduction}
\label{sec:intro}
Probabilistic modelling provides powerful tools for analysing and making future predictions from data. The Bayesian paradigm offers well-calibrated uncertainty estimates on unseen data, by representing the variability of model parameters given the current observations through the posterior distribution. 
However, Bayesian methods are typically computationally expensive, due to the intractability of 
evaluating posterior or marginal probabilities. 
This is especially true for complex models like neural networks, for which a Bayesian would treat all the weight matrices as random variables and integrate them out. Hence approximations have to be applied to overcome this computational intractability in order to make Bayesian methods practical for modern machine learning tasks.

This work considers the problem of amortised inference, 
in which we approximate a given intractable posterior distribution $p$ with a \emph{sampler} $q$, a distribution from which we can draw exact samples. Compared with the typical Monte Carlo (MC) which approximates $p$ with a fixed set of samples, 
the amortised inference distributes the computation cost to training the sampler $q$. 
This allows us to quickly generate a large number of samples at the testing time, and 
can significantly save time and storage when inference is required repeatedly as inner loops of other algorithms, such as in training latent variable models and structured prediction. 

Variational inference \citep{jordan:vi1999} and its stochastic variants \citep{hoffman:svi2013, kingma:vae2014, ranganath:bbvi2014} provide a straightforward approach for amortised inference, 
in which we find an optimal $q$ from a parametric family of distributions $\mathcal Q$ 
by minimizing certain divergence measure (often KL divergence) $\mathrm{D}[q||p]$. 
Unfortunately, except a few very recent attempts \cite{ranganath:ovi2016,wang:amortisedsvgd2016,mescheder:avb2017}, most existing variational approaches require the distributions in $\mathcal Q$ to have  
computationally tractable density functions in order to solve the optimization problem. 
This forms a major restriction on the choice of the approximation set $\mathcal Q$, since 
exact samplers, in the most general form, 
are random variables of form 
$\bm{x} = \bm{f}(\bm{\epsilon})$, where $\bm{f}$ is a (non-linear) transform function, 
and $\bm{\epsilon}$ is some standard distribution such as Gaussian distribution. Except simple cases, e.g.~when $\bm{f}$ is linear, 
it is difficult to explicitly calculate the density function of such random variables.  
%
Therefore, a key challenge is to develop efficient approximate inference algorithms using generic samplers, which we refer as \emph{wild approximations}, without needing to calculate the density functions.
. 
Such algorithms would allow us to 
deploy more flexible families of approximate distributions to obtain better posterior approximations. 

In this paper we develop a new, simple principle for wild variational inference based on \emph{amortising MCMC dynamics}. 
Our method deploys a student-teacher, or actor-critic framework 
to leverage existing MCMC dynamics to a supervisor for training samplers $q$, by iterating the following steps: 
\begin{itemize}
\vspace{-0.05in}
\setlength\itemsep{0.0em}
\item[(1)] the sampler $q$ (student) generates initial samples which are shown to an MCMC sampler;
\item[(2)] the MCMC sampler (teacher) improves the samples by running MCMC transitions;
\item[(3)] the sampler $q$ takes feedback from the teacher and adjust itself in order to generate the improved samples next time. 
\vspace{-0.05in}
\end{itemize}
This framework is highly generic, works for arbitrary sampler families, and can take the advantage of any existing MCMC dynamics for efficient amortised inference. 
Empirical results show that our method works efficiently on Bayesian neural networks and deep generative modelling. 
%
%
%

\section{Background}

\paragraph{Bayesian inference} 
Consider a probabilistic model $p(\bm{x}|\bm{z}, \bm{\theta})$ along with a prior distribution $p_0(\bm{z})$,
where $\bm{x}$ denotes an observed variable, $\bm z$ an 
unknown latent variable, 
and $\bm\theta$ a hyper-parameter that 
is assumed to be given, or will be learned by 
maximizing the marginal likelihood $\log p(\bm x | \bm \theta).$ 
The key computational task of interest is to approximate the posterior distribution of $\bm z$: 
$$
p(\bm{z}|\bm{x}, \bm{\theta}) = \frac{1}{p(\bm{x}|\bm{\theta})} p_0(\bm{z}) p(\bm{x}|\bm{z}, \bm{\theta}), 
~~~~~~~
p(\bm{x}|\bm{\theta}) = \int p_0(\bm{z}) p(\bm{x}|\bm{z}, \bm{\theta}) d\bm z. 
$$
This includes both drawing samples from $p(\bm{z}|\bm{x}, \bm{\theta})$ in order to estimate related average quantities, 
as well as estimating the normalization constant $p(\bm{x}|\bm{\theta})$ for hyper-parameter optimisation or model selection.  
Since both the data $\bm{x}$ and $\bm{\theta}$ are assumed to be fixed in inference, we may drop the dependency on them when it is clear from the context.

\paragraph{MCMC basics}
MCMC provides a powerful, flexible framework for 
drawing (approximate) samples from given distributions.  
An MCMC algorithm is typically specified by its \emph{transition kernel} $\mathcal{K}(\bm{z}' |  \bm{z})$
whose \emph{unique stationary distribution} equals the target distribution $p$ of interest, that is,   
\begin{align}
q = p 
~~~~~~~~~~
\iff
~~~~~~~~~~
q(\bm{z}) = \int q(\bm{z}') \mathcal{K}(\bm{z}|\bm{z}') d\bm{z}', ~~~~\forall \bm{z}.  
\label{eq:mcmc_fixed_point}
\end{align}
This fixed point equation fully characterizes the target distribution $p$,
and hence the 
 inference regarding $p$ can be framed as (approximately) solving equation \eqref{eq:mcmc_fixed_point}. 
 In particular, MCMC algorithms 
can be viewed as stochastic approximations for solving \eqref{eq:mcmc_fixed_point}, 
in which we start with drawing $\bm z_0$ from an initial distribution $q_0$ 
and iteratively draw sample $\bm z_{t}$ at the $t$-th iteration 
from the transition kernel conditioned on the previous state, 
i.e.~
$\bm{z}_t | \bm{z}_{t - 1} \sim \mathcal{K}(\bm{z}_t | \bm{z}_{t-1}).$ 
In this way, the distribution $q_t$ of $\bm z_{t}$ can be viewed as obtained by a fixed point update of form 
\begin{align}\label{eq:qtt}
q_{t}(\bm{z}) \leftarrow  \mathcal K q_{t-1} (\bm z), ~~~~~~\text{where}~~~
 \mathcal K q_{t-1}(\bm z) :=  \int q_{t-1}(\bm{z}') \mathcal{K}(\bm{z}|\bm{z}') d\bm{z}',
\end{align}
so that recursively, we have $q_t = \mathcal K_t q_0$, where $\mathcal K_t$ denotes the $t$-step  transition kernel. 
The standard theory of Markov chains suggests that 
the Markov transition monotonically decreases the KL divergence (\cite{cover:itbook1991}, see also Lemma~1 in the appendix), that is, 
\begin{align}\label{kl}
\mathrm{D_{KL}}[q_{t} ~||~ p] \leq \mathrm{D_{KL}} [q_{t-1} ~||~ p]. 
\end{align}
Therefore, $q_{t}$ converges to the stationary distribution $p$ as $t\to \infty$ under proper conditions.

\section{Amortised MCMC}
\label{sec:ap}
MCMC can be viewed as approximating 
the fixed point update \eqref{eq:qtt} 
in a \emph{non-parametric} fashion, 
returning a set of fixed samples for approximating $p$. 
This motivates amortised MCMC which uses more general parametric approximations of the fixed point update \eqref{eq:qtt} to train \emph{parametric samplers} for amortised inference. 
In the sequel, we introduce our generic framework in Section~\ref{sec:main}, 
discuss in Section~\ref{sec:thechoice} some particular algorithmic choices, 
and apply amortised MCMC to approximate maximum likelihood estimation (MLE) in Section~\ref{approximate}. 
%
\subsection{Main idea: learning to distil MCMC}\label{sec:main} 
Let $\mathcal Q = \{q_{\bm\phi}\}$
be a set of candidate samplers parametrised by $\bm\phi$. 
Our goal is to find an optimal $q_{\bm\phi}$ to closely approximate the posterior distribution $p$ of interest. 
We achieve this by approximating the fixed point update \eqref{eq:qtt}. 
Because of the parametrisation, 
an additional projection step is required to maintain $q$ inside $\mathcal Q$, 
motivating the following update rule at the $t$-th iteration: 
\begin{align}\label{eq:qphi}
\bm \phi_{t} \gets \argmin_{\bm \phi} ~ \mathrm D[ q_{T} ~|| ~  q_{\bm \phi}], ~~~~~~~ 
q_{T} := \mathcal K_T q_{\bm \phi_{t-1}}. 
\end{align}
where $\mathrm D[\cdot || \cdot]$ is 
some divergence measure between distributions
whose choice is discussed in Section~\ref{sec:thechoice}. 
Note that we extended \eqref{eq:qtt} to use the $T$-step transition kernel $\mathcal K_{T}$.  
If $\mathcal Q$ is taken to be large enough so that $\mathcal K_T q_{\phi_{t-1}} \in \mathcal Q$, 
then the projection update \eqref{eq:qphi} (with $T=1$) reduces to \eqref{eq:qtt}. 
In practice, a gradient descent method can be used to solve \eqref{eq:qphi}: 
\begin{equation}
\bm{\phi}_{t} 
\leftarrow \bm{\phi}_{t-1} - \eta  \nabla_{\bm\phi}\mathrm D[ \mathcal K_T q_{\bm \phi_{t-1}} ~|| ~  q_{\bm \phi}]|_{\bm{\phi} = \bm{\phi}_{t-1}}. 
\label{eq:gradient_update_q}
\end{equation}
It is often intractable to evaluate $\nabla_{\bm\phi}\mathrm D[ q_T ~|| ~  q_{\bm \phi}]$ thus an approximation is needed. 
This can be done by approximating 
 $q_{\bm\phi_{t-1}}$ with samples $\{\bm z^{k}_0\}$ drawn from it, 
 and approximating $q_T = \mathcal K_T q_{\bm{\phi}_{t-1}}$
 with sample $\{\bm z^{k}_T\}$ drawn 
 by following the Markov transition $\mathcal K(\cdot | \cdot)$ for $T$ times starting at $\{\bm z^{k}_0\}$. 
These samples are then used to 
estimate the gradient and update $\bm{\phi}_t$  by \eqref{eq:gradient_update_q} 
in order to \emph{``move'' $q_{\bm{\phi}_{t-1}}$ towards $q_T = \mathcal K_T q_{\bm{\phi}_{t-1}}$}, 
which is closer to the target distribution according to \eqref{kl}. 
To summarise, our generic framework requires three main ingredients: 
\begin{itemize}
\vspace{-0.05in}
\setlength\itemsep{0.0em}
	\item[(1)] a parametric set $\mathcal Q = \{q_{\bm \phi}\}$ of the sampler distributions (the student);
	\item[(2)] an MCMC dynamics with kernel $\mathcal{K}(\bm{z}_t | \bm{z}_{t-1})$ (the teacher);
	\item[(3)] a divergence $\mathrm D[\cdot||\cdot]$ and update rule for $\bm \phi$ (the feedback). 
\vspace{-0.05in}
\end{itemize}
By selecting these components tailored to a specific approximate inference task, 
the method provides a highly generic framework, applicable to both continuous and discrete distribution cases, and as we shall see later, extensible to wild approximations without a tractable density.

\paragraph{Remark}
MCMC can be viewed as a special case of our framework, in which the samplers $\mathcal Q$
are empirical distributions parametrised by the MCMC samples $\bm z_t$, that is, $q_{\bm \phi_t}(\bm z)  = \delta(\bm z- {\bm \phi_t})$, ~ 
$\bm{\phi}_t = \bm z_t$, where the sample $\bm z_t$ is treated as the parameter $\bm \phi_t$ in our framework, 
and is the only possible output of the sampler $q_{\bm \phi_t}.$  
Our framework allows more flexible parametrisations of samplers, which significantly saves running time and storage at test time. 

\paragraph{Remark}
The same type of projected fixed point updates as \eqref{eq:qphi} 
have been widely used in reinforcement learning (RL), 
including deep Q learning (DQN) \citep{mnih2013playing}, 
and temporal difference learning with function approximations in general \citep{sutton1998reinforcement}.
In this scenario the Q- or V- networks 
are iteratively adjusted by applying projected fixed point updates of the Bellman equation. 
This provides an opportunity to 
strengthen our method with the vast RL literature. 
For example, similar to the case of RL, 
the convergence of updates of form \eqref{eq:qphi} are not theoretically guaranteed in general, especially when the parametric set $\mathcal Q$ 
 is complex or non-convex.
However, the practical stabilisation tricks developed in the DQN literature \citep{mnih2013playing} can be potentially 
 applied to our case, 
 and theoretical analysis developed in RL \citep{tsitsiklis1997analysis}
 can be borrowed to establish convergence of our method  under simple assumptions (e.g., when $\mathcal Q$ is linear or convex). 


\subsection{The choice of update rule} \label{sec:thechoice}
The choice of the update signal and ways to estimate it plays a crucial role in our framework, 
and it should be carefully selected to ensure both strong discrimination power and 
computational tractability with respect to the parametric family of $q$ that we use. 
We discuss three update rules in the following.
%
\paragraph{KL divergence minimisation} A simple approach is to use the inclusive KL-divergence in \eqref{eq:qphi}: 
$$
\mathrm{D}_{\text{KL}}[q_T~||~q_{\bm\phi}] = \mathbb{E}_{q_T} \left[ \log q_T(\bm{z}|\bm{x}) - \log q_{\bm\phi}(\bm{z}|\bm{x}) \right],
$$
and for the purpose of optimising $\bm\phi$, it only requires an MC estimate of $-\mathbb{E}_{q_T}[\log q_{\bm\phi}(\bm{z}|\bm{x})]$. 
This gives a simple algorithm that is a hybrid of MCMC and VI, and appears to be new to our knowledge. Similar algorithms include the so called cross entropy method \citep{crossentrode2005tutorial} which replaces $q_T$ with an importance weighted distribution, and methods for tuning proposal distributions for sequential Monte Carlo \cite{cornebise:smc2009,gu:nasmc2015}.
\paragraph{Adversarially estimated divergences} 
Unfortunately, the inclusive KL divergence requires the density of the sampler $q_{\bm\phi}$ to be evaluated, 
and can not be used directly for wild approximations.
Instead, we need to estimate the divergences based on samples $\{\bm z_0^k\}\sim q_{\bm\phi}$ and $\{\bm z_T^k\}\sim q_T$. 
To address this, we borrow the 
idea of generative adversarial networks (GAN) \cite{goodfellow:gan2014} to construct a sample-based estimator of the selected divergence. As an example, consider the Jensen-Shannon divergence:
$
\mathrm{D}_{\text{JS}}[q_T||q] = \frac{1}{2} \mathrm{D}_{\text{KL}}[q_T || \tilde{q}] + \frac{1}{2} \mathrm{D}_{\text{KL}}[q || \tilde{q}],
$
with $\tilde{q} = \frac{1}{2} q + \frac{1}{2} q_T$.
Since none of the three distributions have tractable density, a discriminator $d_{\bm{\psi}}(\bm{z}|\bm{x})$ is trained to provide a stochastic lower-bound 
\begin{equation}
\mathrm{D}_{\text{JS}}[q_T||q] \geq \mathrm{D}_{\text{adv}}[\{\bm{z}_T^k \} || \{ \bm{z}_0^k \} ] = \frac{1}{K} \sum_{k=1}^K \log \sigma(d_{\bm{\psi}}(\bm{z}_T^k | \bm{x}))
+ \frac{1}{K} \sum_{k=1}^K \log (1 - \sigma(d_{\bm{\psi}}(\bm{z}_0^k | \bm{x}))),
\label{eq:gan_objective}
\end{equation}
with $\sigma(\cdot)$ the sigmoid function and $\bm{z}_0^k$, $\bm{z}_T^k$ samples from $q$ and $q_T$, respectively. 
Recent work \cite{nowozin:fgan2016} extends adversarial training to $f$-divergences where the two KL-divergences are special cases in that rich family. In such case $\mathrm{D}_{\text{adv}}$ also corresponds to the variational lower-bound to the selected $f$-divergence and the discriminator can also be defined accordingly. Furthermore, the density ratio estimation literature \cite{nguyen2010estimating, sugiyama:ratio2009,sugiyama:ratio2012} suggests that the discriminator $d_{\bm{\psi}}$ in (\ref{eq:gan_objective}) can be used to estimate $\log ({q_T}/{q_{\bm\phi}})$, i.e.~the objective function for $q_{\bm\phi}$ could be decoupled from that for the discriminator \cite{mohamed:gan2016}.

\paragraph{Energy matching} 
An alternative approach to matching
$q_{\bm\phi}$ with 
 $q_T$ is to match their first $M$ moments. 
In particular, matching the mean and variance is equivalent to minimising $\mathrm{D}_{\text{KL}}[q_T || q_{\bm\phi}]$ with $q_T$ fixed and $q_{\bm\phi}$ a Gaussian distribution. However it is difficult to know beforehand which moments are important for accurate approximations. Instead we propose \emph{energy matching}, which matches the expectation of the log of the joint distribution $ p(\bm x, \bm z) \propto p(\bm{z}| \bm{x})$ under $q_\phi$ and $q_T$:
\begin{equation}
\mathrm D_{\text{em}}[q_T ~||~ q_{\bm\phi}]  = \big| \mathbb{E}_{q_T(\bm{z}_T|\bm{x})}\left[ \log p(\bm{x}, \bm{z}_T) \right] - \mathbb{E}_{q_{\vparam}(\bm{z}|\bm{x})}\left[ \log p(\bm{x}, \bm{z}) \right] \big|^{\beta}, \quad \beta > 0,
\label{eq:energy_matching}
\end{equation} 
Although $\mathrm{D}_{\text{em}}[\cdot || \cdot ]$ is not a valid divergence measure since $\mathrm{D}_{\text{em}}[q_T ~||~ q_{\bm\phi}]=0$ does not imply $q_{\bm\phi} = q_T$, 
this construction puts emphasis on moments that are mostly important for accurate predictive distribution approximation, which is still good for inference. 
Furthermore, as (\ref{eq:energy_matching}) can be approximated with Monte Carlo methods which only require samples from $q$, this energy matching objective can be applied to wild variational inference settings as it does not require density evaluation of $q_{\bm\phi}$ and $q_{T}$. Another motivation from contrastive divergence \citep{hinton:cd2002} is also discussed in the appendix.

\subsection{Approximate MLE with amortised MCMC}\label{approximate}
Learning latent variable models have become an important topic with increasing interest. Our amortised inference method can be used to develop more flexible and accurate approximations for learning.  
Consider the variational auto-encoder (VAE) \citep{kingma:vae2014,rezende:vae2014} which approximates maximum likelihood training (MLE) by maximising the following variational lower-bound over the model parameters $\bm{\theta}$ and variational parameters $\bm{\phi}$:
\begin{equation}
\max_{\theta, ~\bm\phi} 
\big\{\mathbb{E}_{q_{\bm\phi}} \left[ \log p(\bm{x} | \bm{z}; \bm{\theta}) \right] - \mathrm{D}_{\text{KL}}[q_{\bm\phi}(\bm{z}|\bm{x})||p_0(\bm{z})]
= \log p(\bm{x}|\bm{\theta}) - \mathrm{D}_{\text{KL}}[q_{\bm\phi}(\bm{z}|\bm{x}) || p(\bm{z}| \bm{x}, \bm{\theta}) ] \big\}. 
\label{eq:vae_mle}
\end{equation}
Our method can be directly used to update $\bm\phi$ in \eqref{eq:vae_mle}. 
This can be done using the inclusive KL divergence if the density  $q_{\bm\phi}$ is tractable, 
and adversarially estimated divergence or energy matching for wild approximation when the density $q_{\bm \phi}$ is intractable.  

Next we turn to the optimisation of the hyper-parameters $\bm{\theta}$ where we decouple their objective function from that of $\bm{\phi}$. 
Because $\mathrm{D}_{\text{KL}}[q || p] \geq \mathrm{D}_{\text{KL}}[q_T || p]$ when $p$ is the stationary distribution of the MCMC, 
the following  objective  forms a tighter lower-bound to the marginal likelihood: 
\begin{equation}
\log p(\bm{x}|\bm{\theta}) - \mathrm{D}_{\text{KL}}[q_{T}(\bm{z}|\bm{x}) || p(\bm{z}| \bm{x}, \bm{\theta}) ] = \mathbb{E}_{q_{T}}[ \log p(\bm{x} | \bm{z}, \bm{\theta}) ] + \text{const of } \bm{\theta}. 
\label{eq:mcmc_mle}
\end{equation}
Empirical evidences \cite{burda:iwae2016,li:vrbound2016} suggested tighter lower-bounds often lead to better results. Monte Carlo estimation is applied to estimate the lower-bound (\ref{eq:mcmc_mle}) with samples $\{ \bm{z}_T^k \} \sim q_T$.

The full method when using adversarially estimated divergences 
is presented in Algorithm \ref{alg:amortised_mcmc_adversarial}, in which we train a discriminator $d_{\bm{\phi}}$ to estimate the selected divergence, and propagate learning signals back through the samples from $q_{\bm\phi}$. 
 Note here the update step for the discriminator and the $q_{\psi}$ could be executed for more iterations in order to achieve better approximations to the current posterior. This strategy turns the algorithm into a stochastic EM with MCMC methods approximating the E-step \cite{celeux:sem1985,celeux:sem1995}. 
RKHS-based and energy-based moments \cite{gretton:mmd2012, sejdinovic2013equivalence} can also be applied as the discrepancy measure in step 2, but this is not explored here. 


\begin{algorithm}[t] 
\begin{algorithmic}[1] 
	\STATE Sample $\bm{z}_{0}^{1}, ..., \bm{z}_{0}^K \sim q_{\bm{\phi}}(\bm{z}|\bm{x})$, and simulate $\bm{z}_{T}^{k} \sim \mathcal K_T(\cdot |\bm{z}_{k}^{0} )$, for $k = 1,\ldots, K$. 
	\STATE Compute $\mathrm{D}_{\text{adv}}[ \{ \bm{z}_T^k \}  || \{ \bm{z}_0^k \} ]$ using discriminator $\bm\psi$.
	\STATE Update $\bm{\phi}$ and $\bm{\psi}$ by 1-step gradient descent/ascent.
	\STATE If learning $\bm{\theta}$: compute 1-step gradient ascent with
	$ \nabla_{\bm{\theta}} \frac{1}{K} \sum_{k=1}^K \log p(\bm{\bm{x}} | \bm{z}_T^k, \bm{\theta}).$
\end{algorithmic}
\caption{Amortised MCMC with adversarially estimated divergences (one update iteration)}
\label{alg:amortised_mcmc_adversarial} 
\end{algorithm}



\section{Related Work}\label{related}
Since \citep{goodfellow:gan2014}, generative adversarial networks (GANs) have attracted large attention from both academia and industry. 
We should distinguish the problem scope of our work with that of GANs: 
amortised MCMC aims to match an (implicitly defined) $q$ to the posterior \emph{distribution $p$}, 
while GANs aim to match a $q$ with an observed \emph{sample}, which we leverage as an inner loop for the divergence minimisation. Hence our framework could also benefit from the recent advances in this area \cite{nowozin:fgan2016,mohamed:gan2016}.

The amortisation framework is in similar spirit to \cite{snelson:compact2005,korattikara:dark2015} in that both approaches ``distil'' an MCMC sampler with a parametric model. Unlike the presented framework, \cite{snelson:compact2005} and \cite{korattikara:dark2015} used a student model to approximate the predictive likelihood, and that student model is not used to initialise the MCMC transitions. We believe that initialising MCMC with the student model is important in amortising dynamics, as the teacher can ``monitor'' the student's progress and provide learning signals tailored to the student's need. Moreover, since the initialisation is improved after each student update, the quality of the teacher's samples also improves. Another related, but different approach \cite{rasmussen:gp_hmc2003} considered speeding-up Hybrid Monte Carlo \citep{neal:mcmc2011} by approximating the transition kernel using a Gaussian process. Amortised MCMC could benefit from this line of work if the MCMC updates are too expensive.  

Perhaps the most related approaches to our framework (in the sense of using $q$ of flexible forms) are operator variational inference (OPVI, \citep{ranganath:ovi2016}), amortised SVGD  \citep{wang:amortisedsvgd2016}, and adversarial variational Bayes (AVB \citep{mescheder:avb2017}, also concurrently proposed by \citep{li:wild2016,huszar:implicit2017,tran:implicit2017}). These works assumed the $q_{\bm\phi}$ distribution to be represented by a neural network warping input noise. OPVI minimises the Stein discrepancy \cite{stein:stein_method1972} between the exact and approximate posterior, where the optimal test function is determined by optimising a discriminator. 
Though theoretically appealing, this method seems still impractical for large scale problems. 
Amortised SVGD can be viewed as a special case of our framework, which specifically uses a deterministic Stein variational gradient dynamic \cite{liu:svgd2016} and an $l_2$-norm as the divergence measure.  
AVB estimates the KL-divergence $\mathrm{D}_{\text{KL}}[q||p_0]$ in the variational lower-bound (\ref{eq:vae_mle}) with GAN and density ratio estimation, making it closely related to the adversarial auto-encoder \cite{makhzani:adversarial_ae2015}. However we conjecture that the main learning signal of AVB comes from the ``reconstruction error'' term $\mathbb{E}_q[\log p(\bm{x}|\bm{z}, \bm{\theta})]$, and the regularisation power strongly depends on the adversarial estimation of $\mathrm{D}_{\text{KL}}[q||p_0]$, which can be weak as the discriminator is non-optimal in almost all cases.

\section{Experiments}
We evaluate amortised MCMC with both toy and real-world examples. For simplicity we refer the proposed framework as AMC. In the appendix experimental settings are further presented. Code will be released at \url{https://github.com/FirstAuthor/AmortisedMCMC}. 

\subsection{Synthetic example: fitting a mixture of Gaussians}

We first consider fitting a Gaussian mixture 
$p(z) = \frac{1}{2} \mathcal{N}(z; -3, 1) + \frac{1}{2}\mathcal{N}(z; 3, 1)$ with the variational program proposed by \cite{ranganath:ovi2016} as the following:
$\epsilon_1, \epsilon_2, \epsilon_3 \sim \mathcal{N}(\epsilon; 0, 1)$, $z = \mathbbm{1}_{\epsilon_3 \geq 0} \text{ReLU}(w_1 \epsilon_1 + b_1) - \mathbbm{1}_{\epsilon_3 < 0} \text{ReLU}(w_2 \epsilon_2 + b_2)$. We further tested a small multi-layer perceptron (MLP) model of size [3, 20, 20, 1] which warps $\bm{\epsilon}$ to generate the samples. 
The Jensen-Shannon divergence is adversarially estimated with an MLP of the same architecture. The MCMC sampler is Langevin dynamics with rejection (MALA \cite{roberts:mala1998}), and in training 10 parallel chains are used.
The fitted approximations are visualised in Figure \ref{fig:mog_toy}. Both models cover both modes, however the variational program performs better in terms of estimating the variance of each Gaussian component. Thus an intelligent design of the $q$ network can achieve better performance with much fewer number of parameters.

We empirically investigate the effect of the chain length $T$ on the approximation quality using the MLP approximation. Time steps $T=1, 5, 10$ with step-sizes $\eta = 0.1, 0.02, 0.01$ are tested (each repeating 10 times), where by making $T\eta = 1.0$ a constant, the particles approximately move equal distances during MCMC transitions. In Figure \ref{fig:mog_toy} we shown the Kernel Stein Discrepancy (KSD \citep{liu:ksd2016}) as a metric of approximation error. With small chain length the student quickly learns the posterior shape, but running more MCMC transitions results in better approximation accuracy. A potential way to balance the time-accuracy trade-off is to initially use short Markov chains initially for AMC, but to lengthen them as AMC converges. This strategy has been widely applied to contrastive-divergence like methods \cite{hinton:cd2002,salakhutdinov:ais2008}. We leave the exploration of this idea to future work.

\begin{figure}[t]
 \vspace{-0.1in}
 \centering
 \includegraphics[width=0.9\linewidth]{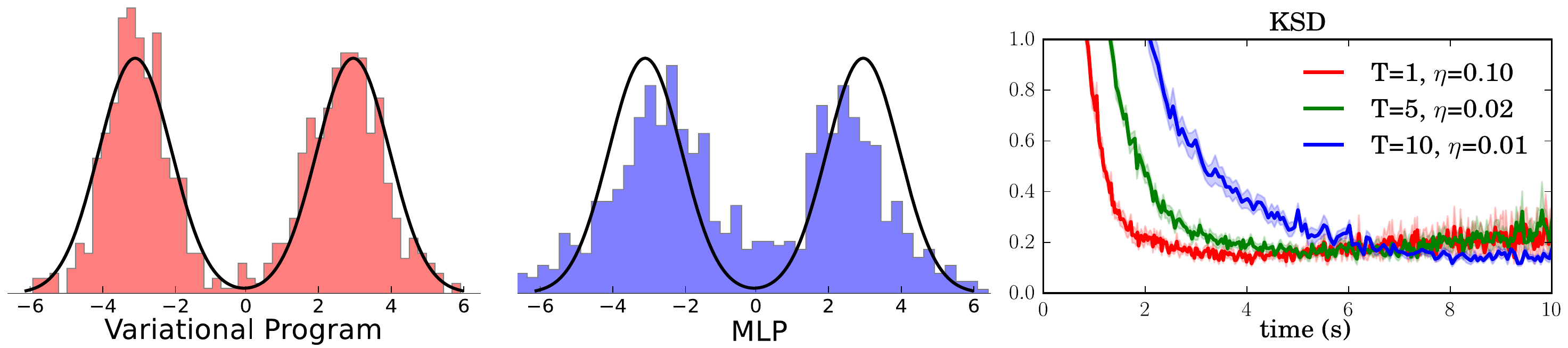}
 \caption{Approximating a Gaussian mixture density.}
 \label{fig:mog_toy}
 \vspace{-0.1in}
\end{figure}

\subsection{Bayesian neural network classification}

Next we apply amortised MCMC to classification using Bayesian neural networks. Here the random variable $\bm{z}$ denotes the neural network weights which could be thousands of dimensions. In such high dimensions a large number of MCMC samples are typically required for good approximation. Instead we consider AMC as an alternative, which allows us to use much fewer samples during training as it decouples the samples from evaluation, leading to massive savings of time and storage. To validate this, we take from the UCI repository \citep{lichman:uci2013} 7 binary classification datasets and train a 50-unit Bayesian neural network. For comparison we test the mean-field Gaussian approximation trained by VI with 10 samples, and MALA with 100 particles simulated and stored. The approximated posterior for AMC is constructed by first taking a mean-field Gaussian approximation, then normalising the $K=10$ samples by their empirical mean and variance. This wild approximation is trained with the energy matching objective with $\beta = 2$, where we also use MALA as the dynamics with $T=1$.

The results are reported in Table \ref{tab:bnn_results}. For test log-likelihood MALA is generally better than VI as expected. More importantly, AMC performs similarly, and for some datasets even better, than MALA. VI returns the best test error metrics on three of the datasets, with AMC and MALA topped for the rest. In order to demonstrate the speed-accuracy trade-off we visualise in Figure \ref{fig:bnn_time} the (negative) test-LL and error as a function of CPU time. In this case we further test MALA with 10 samples, which is much faster than the 100 samples version, but it pays the price of slightly worse results. However AMC achieves better accuracy in a smaller amount of time, and in general out-performs the 10-sample MALA. These observations show that AMC with energy matching can be used to train Bayesian neural works and achieves a balance between computational complexity and performance.

\begin{figure}[t]
\centering
\captionof{table}{BNN classification experiment results}
\label{tab:bnn_results}
\resizebox{\textwidth}{!}{%
\begin{tabular}{l@{\ica}|r@{$\pm$}l@{\ica}r@{$\pm$}l@{\ica}r@{$\pm$}l@{\ica}|r@{$\pm$}l@{\ica}r@{$\pm$}l@{\ica}r@{$\pm$}l@{\ica}}
\hline
&\multicolumn{6}{c}{Average Test Log-likelihood}&\multicolumn{6}{c}{Average Test Error}\\
\bf{Dataset}&\multicolumn{2}{c}{\bf{VI+Gaussian}}&\multicolumn{2}{c}{\bf{AMC}}&\multicolumn{2}{c|}{\bf{MALA}}&\multicolumn{2}{c}{\bf{VI+Gaussian}}&\multicolumn{2}{c}{\bf{AMC}}&\multicolumn{2}{c}{\bf{MALA}}\\
\hline
australian&\textbf{-0.633}&\textbf{0.008}&-0.666&0.015&-0.636&0.010&\textbf{0.315}&\textbf{0.014}&0.360&0.011&0.344&0.013\\
breast&-0.096&0.010&\textbf{-0.091}&\textbf{0.008}&-0.094&0.010&\textbf{0.029}&\textbf{0.003}&0.030&0.004&0.037&0.004\\
colon&-0.799&0.246&-0.491&0.104&\textbf{-0.420}&\textbf{0.027}&\textbf{0.125}&\textbf{0.023}&0.167&0.031&0.167&0.026\\
crabs&-0.221&0.012&\textbf{-0.115}&\textbf{0.011}&-0.179&0.010&0.070&0.013&0.040&0.010&\textbf{0.035}&\textbf{0.010}\\
ionosphere&-0.241&0.019&-0.230&0.031&\textbf{-0.179}&\textbf{0.013}&0.099&0.011&0.077&0.013&\textbf{0.064}&\textbf{0.010}\\
pima&-0.503&0.010&-0.506&0.013&\textbf{-0.498}&\textbf{0.012}&0.262&0.008&\textbf{0.245}&\textbf{0.008}&0.247&0.008\\
sonar&-0.389&0.025&\textbf{-0.347}&\textbf{0.030}&-0.366&0.021&0.179&0.014&\textbf{0.150}&\textbf{0.016}&0.171&0.020\\
\hline
\end{tabular}%
\vspace{0.1in}
}

\centering
\includegraphics[width=0.85\linewidth]{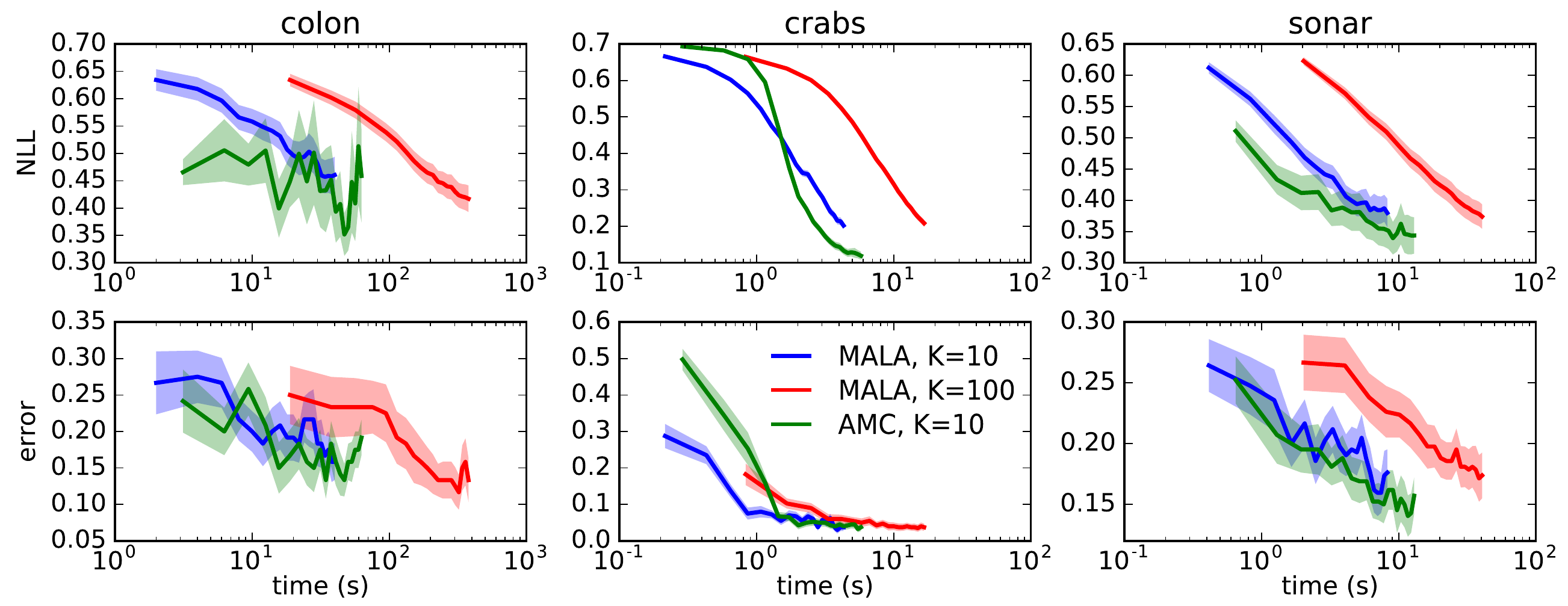}
\label{fig:bnn_time}
\captionof{figure}{Running time/performance trade-off. Time reported in log scale. See main text for details.}
\vspace{-0.1in}
\end{figure}

\subsection{Deep generative models}
The final experiment considers training deep generative models on the dynamically binarised MNIST dataset, containing 60,000 training datapoints and 10,000 test images \cite{burda:iwae2016}. For benchmark a convolutional VAE with $\text{dim}(\bm{z}) = 32$ latent variables is tested. The Gaussian encoder consists of a convolutional network with $5 \times 5$ filters, stride 2 and [16, 32, 32] feature maps, followed by a fully connected network of size [500, 32 $\times$ 2]. The generative model has a symmetric architecture but with stride convolution replaced by deconvolution layers. This generative model architecture is fixed for all the tests. We also test AMC with inclusive KL divergence on Gaussian encoders, and compare to the naive approach which trains the encoder by maximising variational lower-bound (MCMC-VI). 

We construct two non-Gaussian encoders for AVB and AMC (see appendix). Both encoders start from a CNN followed by a reshaping operation. Then the first model (CNN-G) splits the CNN's output vector into $[\bm{h}(\bm{x}), \bm{\mu}(\bm{x}), \log \bm{\sigma}(\bm{x})]$, samples a Gaussian noise $\bm{\epsilon} \sim \mathcal{N}(\bm{\epsilon}; \bm{\mu}(\bm{x}), \text{diag}[\bm{\sigma}^2(\bm{x})])$, and feeds $[\bm{h}(\bm{x}), \bm{\epsilon}]$ to an MLP for generating $\bm{z}$.
The second encoder (CNN-B) simply applies multiplicative Bernoulli noise with dropout rate 0.5 to the CNN output vector, and uses the same MLP architecture as CNN-G. The discriminator consists of a CNN acting on the input image $\bm{x}$ only, and an MLP acting on both $\bm{z}$ and the CNN output. Batch normalisation is applied to non-Gaussian encoders and the discriminator. The learning rate of Adam \cite{kingma:adam2015} is tuned on the last 5000 training images. Rejection steps are not used for Langevin dynamics as we found it slows down the learning.

\begin{figure}[t]
\begin{minipage}{0.43\linewidth}
\includegraphics[width=0.9\linewidth]{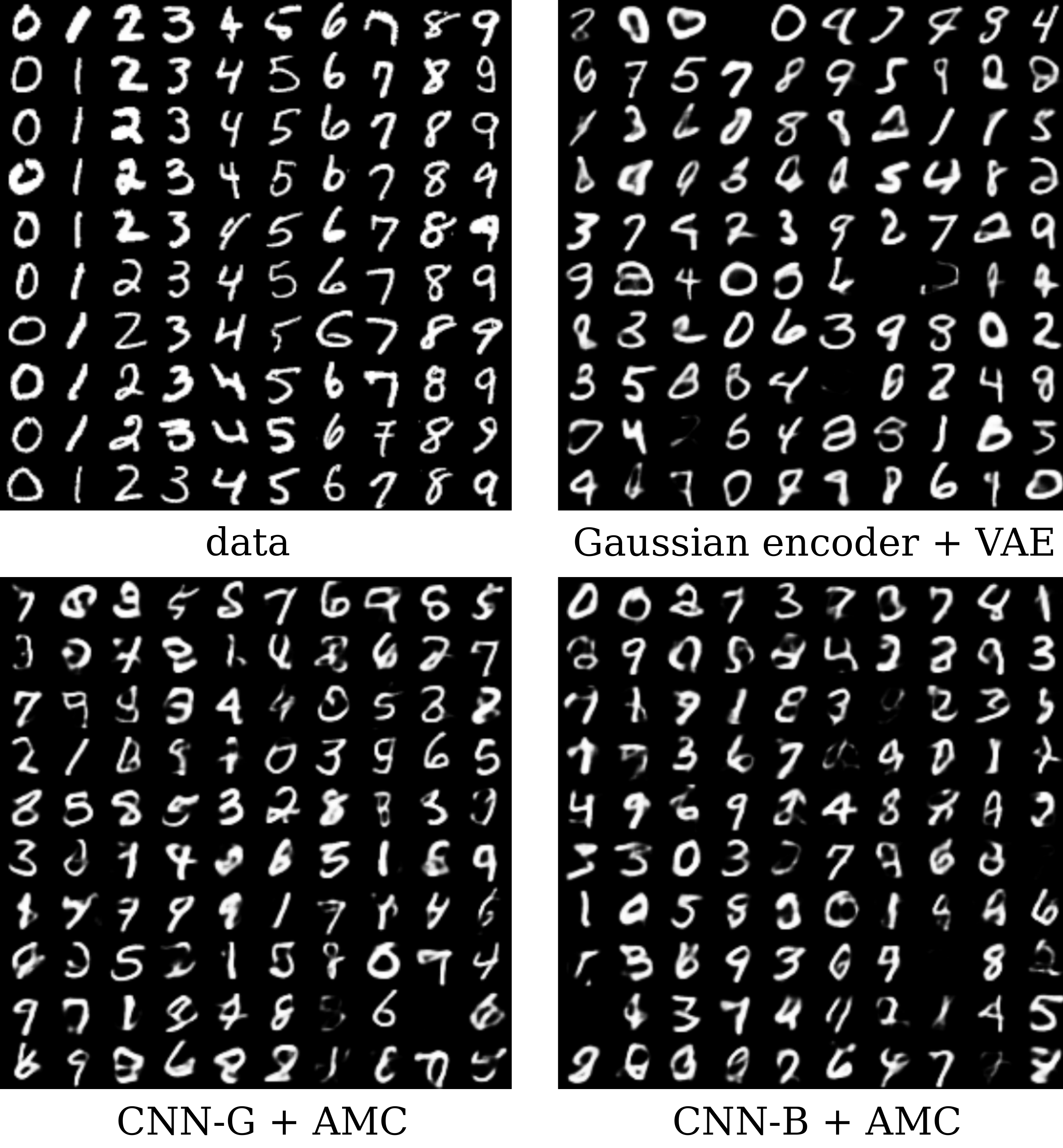}
\captionof{figure}{Generated images.}
\label{fig:samples}
\end{minipage}
\begin{minipage}{0.55\linewidth}
\centering
\includegraphics[width=0.75\linewidth]{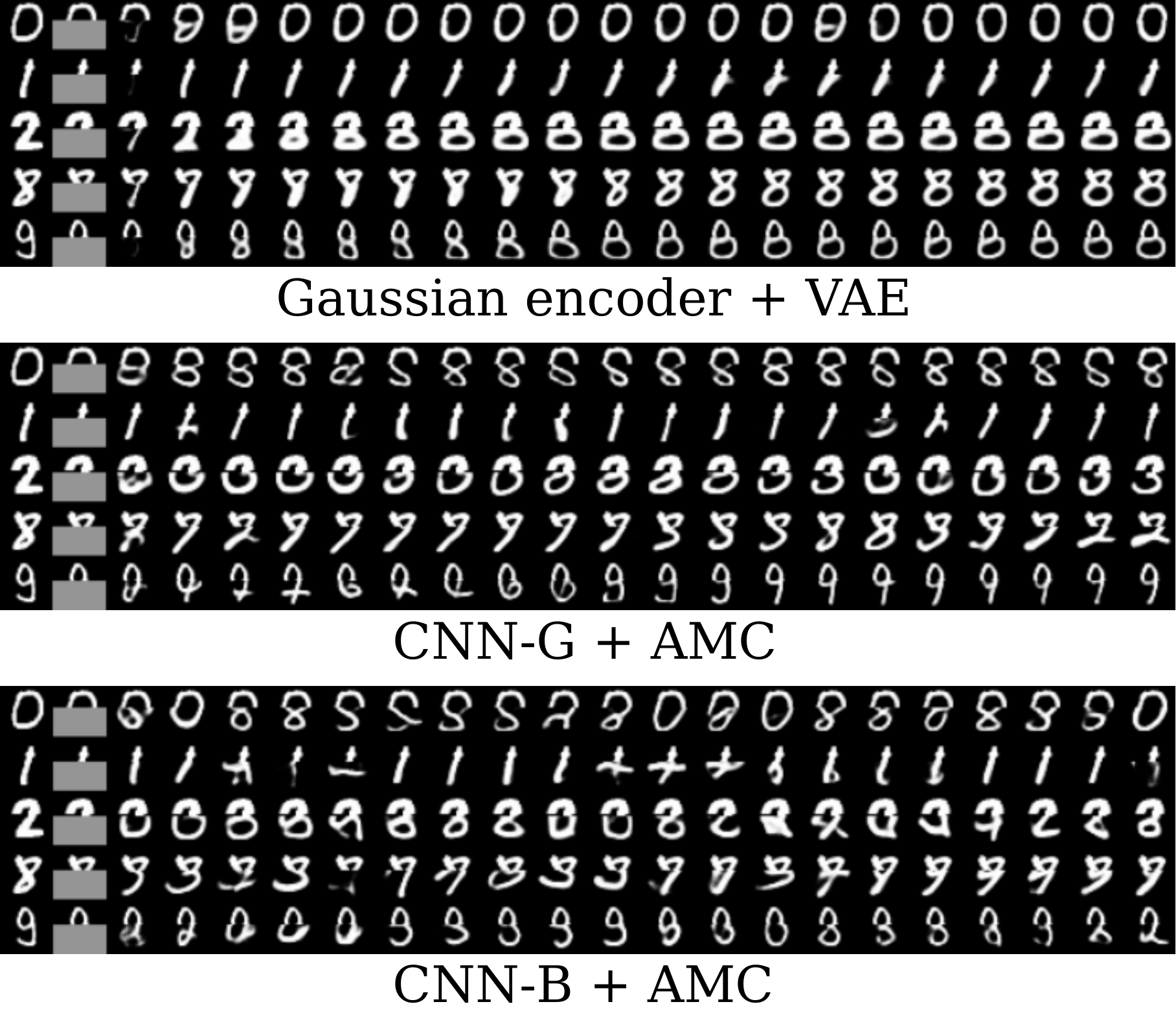}
\captionof{figure}{Imputation results.}
\label{fig:imputation}
\captionof{table}{ Label entropy on nearest neighbours.}
\centering
\label{tab:entropy_imputation}
\scalebox{0.9}{
\begin{tabular}{l@{\ica}r@{$\pm$}l@{\ica}r@{$\pm$}l@{\ica}r@{$\pm$}l@{\ica}r@{$\pm$}l@{\ica}r@{$\pm$}l@{\ica}r@{$\pm$}l@{\ica}r@{$\pm$}}\hline 
\bf{Dataset}&\multicolumn{2}{c}{\bf{ VAE }}&\multicolumn{2}{c}{\bf{ CNN-G }}&\multicolumn{2}{c}{\bf{ CNN-B }} \\ \hline 
Entropy&0.411&0.039&0.701&0.048&0.933&0.049\\ 
$l_1$-norm&0.061&0.000&0.059&0.000&0.064&0.000\\  
\hline 
\vspace{-0.1in}
\end{tabular} 
}
\end{minipage}
\centering
\captionof{table}{ Average Test log-likelihood (LL/nats). }
\label{tab:vae_ll}
\scalebox{0.9}{
\begin{tabular}{lccccc}
\hline
\bf{Encoder} &\bf{Method} & IW-LL & IW-ESS & HAIS-LL & HAIS-ESS \\
\hline
Gaussian & VAE & \bf{-81.31} & \bf{104.11} & -80.64 & 91.59 \\
         & MCMC-VI, $T=5$, $\eta=0.2$ & -90.06 & 110.58 & -89.79 & 85.63 \\
         & AMC, $T=5$, $\eta=0.2$ & -90.71 & 49.02 & -89.64 & 87.93 \\
\hline
CNN-G & AMC, $T=5$, $\eta = 0.2$ & -90.84 & 31.60 & -89.35 & 87.49 \\
      & AMC, $T=50$, $\eta = 0.02$ & -83.30 & 6.84 & \bf{-78.23} & \bf{77.78} \\
      & AVB & -94.97 & 11.30 & -85.92 & 57.21 \\
\hline
CNN-B & AMC, $T=5$, $\eta = 0.2$ & -90.75 & 34.17 & -89.42 & 88.10\\
      & AMC, $T=50$, $\eta = 0.02$ & -83.62 & 8.88 & -80.03 & 80.71\\
      & AVB & -89.47 & 8.98 & -82.66 & 76.90 \\
\hline
N/A & persistent MCMC, $T=50$, $\eta=0.02$ & -84.43 & 9.14 & -78.88 & 77.29 \\
\hline
\end{tabular}
}

\vspace{-0.1in}
\end{figure}

\paragraph{Test Log-likelihood Results}
We report the test log-likelihood (LL) results in Table \ref{tab:vae_ll}. We first follow \cite{burda:iwae2016} to estimate the test log-likelihood with importance sampling using 5000 samples, and for the non-Gaussian encoders we train another Gaussian encoder with VI as the proposal distribution.
VAE appears to be the best method by this metric, and the best AMC model is about 2nats behind. However, effective sample size results (IW-ESS) show that the estimation results for AMC and AVB are unreliable.
Indeed approximate MLE using the variational lower-bound biases the generative network towards the model whose exact posterior is close to the inference network $q$ \cite{turner:two_problems2011}. As the MCMC-guided approximate MLE trains the generative model with $q_T$ (which could be highly non-Gaussian), the VI-fitted Gaussian proposal, employed in the IWAE, can under-estimates the true test log-likelihood by a significant amount.
To verify the conjecture, we estimate the test-LL again but using Hamiltonian annealed importance sampling (HAIS) as suggested by \citep{wu:quantitative2016}. We randomly select 1,000 test images for evaluation, and run HAIS with 10,000 intermediate steps and 100 parallel chains. Estimation results demonstrate that IW-LL significantly under-estimates the test LL for models trained by wild approximations. In this metric the CNN-G model with $T=50$ performs the best, which is significantly better than benchmark VAE. To demonstrate the improvement brought by the wild approximation we further train a generative model with ``persistent MCMC'', by initialising the Markov chain with previous samples and ignoring the posterior changes. The HAIS-LL results shows that out best model is about $0.6$nats better, which is a significant improvement on MNIST.

Although test log-likelihood is an important measures of model quality, \citet{theis:eval2016} has shown that this metric is often largely orthogonal to one that tests visual fidelity when the data is high dimensional. We visualise the generated images in Figure \ref{fig:samples}, and we see that AMC trained models produce samples of similar quality to VAE samples. 

\paragraph{Missing Data Imputation}
We also consider missing data imputation with pixels missing from contiguous sections of the image, i.e.~not at random. 
We follow \cite{rezende:vae2014} using an approximate Gibbs sampling procedure for imputation. With observed and missing pixels denoted as $\bm{x}_o$ and $\bm{x}_m$, the approximate sampling procedure iteratively applies the following transition steps: (1) sample $\bm{z} \sim q(\bm{z}|\bm{x}_o, \bm{x}_m)$ given the imputation $\bm{x}_m$, and (2) sample $\bm{x}^* \sim p(\bm{x}^*|\bm{z}, \bm{\theta})$ and set $\bm{x}_m \leftarrow \bm{x}_m^*$.
In other words, the encoder $q(\bm{z}|\bm{x})$ is used to approximately generate samples from the exact posterior. As ambiguity exists, the exact conditional distribution $p(\bm{x}_m|\bm{x}_o, \bm{\theta})$ is expected to be multi-modal.

Figure \ref{fig:imputation} visualises the imputed images, where starting from the third column the remaining ones show samples for every 2 Gibbs steps. Clearly the approximate Gibbs sampling for VAE is trapped in local modes due to the uni-modal approximation $q$ to the exact posterior. On the other hand, models trained with AMC return diverse imputations that explore the space of compatible images quickly, for instance, CNN-B returns imputations for digit ``9'' with answers 2, 0, 9, 3 and 8. To quantify this, we simulate the approximate Gibbs sampling for $T=100$ steps on the first 100 test images (10 for each class), find the nearest neighbour (in $l_1$-norm) of the imputations in the training dataset, and compute the entropy on the label distribution over these training images. The entropy values and the average $l_1$-distance to the nearest neighbours (divided by the number of pixels $\text{dim}(\bm{x}) = 784$) are presented in Table \ref{tab:entropy_imputation}. These metrics indicate that AMC trained models generate more diverse imputations compared to VAE, yet these imputed images are about the same distance from the training data.

\section{Conclusion and Future Work}
We have proposed an MCMC amortisation algorithm which deploys a student-teacher framework to learn the approximate posterior. By using adversarially estimated divergences and energy matching, the algorithm allows approximations of arbitrary form to be learned. Experiments on Bayesian neural network classification showed that the amortisation method can be used as an alternative to MCMC when computational resources are limited. Application to training deep generative networks returned models that could generate high quality images, and the learned approximation captured multi-modality in generation.
Future work should cover both theoretical and practical directions. Convergence of the amortisation algorithm will be studied. Flexible approximations will be designed to capture multi-modality. Efficient MCMC samplers should be applied to speed-up the fitting process. Practical algorithms for approximating discrete distributions will be further developed. 



\small
\bibliography{references}
\bibliographystyle{plain}
\newpage
\appendix
\section{Examples of wild approximations}
We provide several examples of wild approximations in the following.

\begin{example}
\label{ex:reparam}
(Deterministic transform) 
\emph{
Sampling $\bm{z} \sim q(\bm{z}|\bm{x})$ is defined by first sampling some random noise $\bm{\epsilon} \sim p(\bm{\epsilon})$, then transforming it with a deterministic mapping $\bm{z} = \bm{f}(\bm{\epsilon}, \bm{x})$, which might be defined by a (deep) neural network. These distributions are also called \emph{variational programs} in \citep{ranganath:ovi2016}, or \emph{implicit models} in the generative model context \cite{mohamed:gan2016}. An important note here is that $\bm{f}$ might not be invertible, which differs from the invertible transform techniques discussed in \cite{rezende:flow2015,kingma:iaf2016}.}
\end{example}
\begin{example}
(Truncated Markov chain)
\emph{
Here the samples $\bm{z} \sim q(\bm{z}|\bm{x})$ are defined by finite-step transitions of a Markov chain. Examples include Gibbs sampling in contrastive divergence \citep{hinton:cd2002}, or finite-step simulation of an SG-MCMC algorithm such as SGLD \citep{welling:sgld2011}. It has been shown in \citep{maclaurin:sgd2016, mandt:sgd2016} that the trajectory of SGD can be viewed as a variational approximation to the exact posterior. In these examples the variational parameters are the parameters of the transition kernel, e.g.~step-sizes and/or preconditioning matrices. Related work includes \citet{salimans:mcmcvi2015} which integrates MCMC into VI objective. These methods are more expensive as they require evaluations of $\nabla_{\bm{z}} \log p(\bm{x}, \bm{z}| \bm{\theta})$, but they can be much cheaper than sampling from the exact posterior.
}
\end{example}

\begin{example}
(Stochastic regularisation techniques (SRT))
\emph{
SRT for deep neural network training, e.g.~dropout \citep{srivastava:dropout2014} and related variants \citep{wan:dropconnect2013, singh:swapout2016}, have been re-interpreted as a variational inference method for network weights $\bm{z} = \{ \bm{W} \}$ \citep{gal:dropout2016, gal:uncertainty2016}. The variational parameters $\bm{\phi} = \{ \bm{M} \}$ are the weight matrices of a Bayesian neural network without SRT. The output is computed as $\bm{h} = \sigma((\bm{\epsilon} \odot \bm{x}) \bm{M})$, with $\sigma(\cdot)$ the activation function and $\bm{\epsilon}$ some randomness. This is equivalent to setting $\bm{W} = \text{diag}(\bm{\epsilon}) \bm{M}$, making SRT a special case of example \ref{ex:reparam}. Fast evaluation of $q(\bm{z}|\bm{x})$ during training is intractable as different noise values $\bm{\epsilon}$ are sampled for different inputs in a mini-batch. This means multiple sets of weights are processed if we were to evaluate the density, which has prohibitive costs especially when the network is wide and deep.
}
\end{example}

\section{Approximate MLE with MCMC: mathematical details}
In the main text we stated that $\mathrm{D}_{\text{KL}}[q||p] \geq \mathrm{D}_{\text{KL}}[q_T||p]$ if $q$ is the stationary distribution of the kernel $\mathcal{K}$. This is a direct result of the following lemma, and we provide a proof from \cite{cover:itbook1991} for completeness.

\begin{lemma}\label{lem:dd}
\cite{cover:itbook1991}
Let $q$ and $r$ be two distributions for $\bm{z}_0$. Let $q_t$ and $r_t$ be the corresponded distributions of state $\bm{z}_t$ at time $t$, induced by the transition kernel $\mathcal{K}$. Then $\mathrm{D}_{\text{KL}}[q_t||r_t] \geq \mathrm{D}_{\text{KL}}[q_{t+1}||r_{t+1}]$ for all $t \geq 0$.
\end{lemma}
\begin{proof}
\begin{equation*}
\begin{aligned}
\mathrm{D}_{\text{KL}}[q_t||r_t] &= \mathbb{E}_{q_t}\left[ \log \frac{q_t(\bm{z}_t)}{r_t(\bm{z_t})} \right] \\
&= \mathbb{E}_{q_t(\bm{z}_t) \mathcal{K}(\bm{z}_{t+1}|\bm{z}_t) }\left[ \log \frac{q_t(\bm{z}_t) \mathcal{K}(\bm{z}_{t+1}|\bm{z}_t) }{r_t(\bm{z}_t) \mathcal{K}(\bm{z}_{t+1}|\bm{z}_t)} \right] \\
&= \mathbb{E}_{q_{t+1}(\bm{z}_{t+1}) q_{t+1}(\bm{z}_t|\bm{z}_{t+1})}\left[ \log \frac{q_{t+1}(\bm{z}_{t+1}) q(\bm{z}_{t}|\bm{z}_{t+1}) }{r_{t+1}(\bm{z}_{t+1}) r(\bm{z}_{t}|\bm{z}_{t+1})} \right] \\
&= \mathrm{D}_{\text{KL}}[q_{t+1}||r_{t+1}] + \mathbb{E}_{q_{t+1}} \mathrm{D}_{\text{KL}}[q_{t+1}(\bm{z}_t | \bm{z}_{t+1}) || r_{t+1}(\bm{z}_t | \bm{z}_{t+1})].
\end{aligned}
\end{equation*}
\end{proof}
\vspace{-0.2in}

\newcommand{\ff}{{\bm\phi}}
\section{Energy Matching and Contrastive Divergence}
The energy matching method in Section~\ref{sec:thechoice} can also be roughly motivated by contrastive divergence \citep{hinton:cd2002} as follows. 
First define 
$$\Delta_{\text{CD}}[q_\ff ~||~p]  :=\mathrm D_\mathrm{KL}[q_\ff ~||~p] - \mathrm D_\mathrm{KL}[ q_{T} ||p],$$
where $q_{T} =  \mathcal K_T q_{\ff}$. 
From Lemma~\ref{lem:dd}, 
$\Delta_{\text{CD}}[q_\ff ~||~p]  \geq 0$ and can be used as a minimisation objective to fit an approximate posterior $q$. 
Expanding this contrastive divergence equation, we have:
\begin{equation}
\begin{aligned}
\Delta_{\text{CD}}[q_\ff ~||~p]  
	&= \mathbb{E}_{q_{T}(\bm{z}_T|\bm{x})} \left[ \log p(\bm{x}, \bm{z}_T)  \right] - \mathbb{E}_{q_\ff (\bm{z}|\bm{x})} \left[ \log p(\bm{x}, \bm{z})  \right]  ~+~ R , 
\end{aligned}
\label{eq:cd}
\end{equation}
where 
$$
R := \mathbb{E}_{q_{T}(\bm{z}|\bm{x}) \mathcal{K}_T(\bm{z}_T | \bm{z})} \left[ \log \frac{ q_{T}(\bm{z} | \bm{z}_T, \bm{x})}{\mathcal{K}_T(\bm{z}_T | \bm{z})} \right],
$$
and $q_{T}(\bm{z} | \bm{z}_T, \bm{x}) = q_{\ff}(\bm{z}|\bm{x}) \mathcal{K}_T(\bm{z}_T|\bm{z}) / q_{\ff T}(\bm{z}_T|\bm{x})$ is the ``posterior'' of $\bm{z}$ given the sample $\bm{z}_T$ after $T$-step MCMC. We ignore the residual term $R$ in when $T$ is small as now $\z$ and $\z_T$ are highly correlated, which motivates the energy matching objective (\ref{eq:energy_matching}). For large $T$ one can also use density ratio estimations methods to estimate the third term, but we leave the investigation to future work.

\section{Uncorrelated versus correlated Simulations}

In our algorithm, we generate $\{\bm z_T^k\}$ by simulating Markov transition for $T$ steps starting from $\{\bm z_0^k\} \sim q$, and use  these two samples 
$\{\bm z_0^k\}$ and $\{\bm z_T^k\}$ to estimate the divergence $D[q_T ||q]$.  
However, note that 
$\{\bm z_0^k\}$ and $\{\bm z_T^k\}$ is corrected, and this may introduce bias in the divergence estimation. A way to eliminate this correction is to simulate another $\{\bm  z_0^k\}$ independently and use it to replace the original sample. 
In practice, we find that the correlated/uncorrelated samples exhibit different behaviour during training. We consider the extreme case $K=1$ and small $T$ as an example. Using correlated samples would cause the teacher and the student's samples remaining in the same mode with high probability and thus easily confuse the discriminator and the student (generator) improves fast. On the other hand, if $\bm{z}_T$ is simulated from a Markov chain independent with $\bm{z}_0$, then these samples might be far away from each other (especially when $q$ is forced to be multi-modal), hence the discriminator can easily get saturated, providing no learning signal to the student. 
The above problem could potentially be solved using advanced techniques, e.g.~Wasserstein GAN \cite{arjovsky:wgan2017} which proposed minimising (an adversarial estimate of) Wasserstein-1 distance. In that case the gradient of $q$ won't saturate even when the two sets of samples are separated. But minimising Wasserstein distance would fit the $q$ distribution to the posterior in an ``optimal transport'' way, which presumably prefers moving the $q$ samples to their nearest modes in the exact posterior. 

\section{Experimental details}
\subsection{Bayesian neural networks: settings}
We use a one-hidden-layer neural network with 50 hidden units and ReLU activation. The mini-batch size is 32 and the learning rate is 0.001. The step-size of MALA is adaptively adjusted to achieve acceptance rate 0.99. The experiments are repeated on 20 random splits of the datasets. We time all the tests on a desktop machine with Intel(R) Core(TM) i7-4930K CPU @ 3.40GHz.

\subsection{Deep generative models: settings}
We construct two non-Gaussian encoders for tests of AVB and AMC. Both encoders start from a CNN with $3 \times 3$ filters, stride 2 and [32, 64, 128, 256] feature maps, followed by a reshaping operation. Then the first model (CNN-G) splits the output vector of the CNN into $[\bm{h}(\bm{x}), \bm{\mu}(\bm{x}), \log \bm{\sigma}(\bm{x})]$, samples a Gaussian noise variable of 32 dimensions $\bm{\epsilon} \sim \mathcal{N}(\bm{\epsilon}; \bm{\mu}(\bm{x}), \text{diag}[\bm{\sigma}^2(\bm{x})])$, and feeds $[\bm{h}(\bm{x}), \bm{\epsilon}]$ to an MLP which has hidden layer sizes [500, 500, 32].
The second encoder (CNN-B) simply applies multiplicative Bernoulli noise with dropout rate 0.5 to the CNN output vector, and uses the same MLP architecture as CNN-G. A discriminator is trained for AVB and AMC methods, with a CNN (of the same architecture as the encoders) acting on the input image $\bm{x}$ only, and a MLP of [32+1024, 500, 500, 1] layers. All networks use leaky ReLU activation functions with slope parameter 0.2, except for the output of the deconvolution network which uses sigmoid activation. Batch normalisation is applied to non-Gaussian encoders and the discriminator. The Adam optimiser \cite{kingma:adam2015} is used with learning rates tuned on the last 5000 training images. Rejection steps are not used in the Langevin dynamics as we found this slows down the learning.

\section{More visualisation results}
\begin{figure}[t]
 \centering
 \subfigure[Gaussian encoder + VAE]{
 \includegraphics[width=0.65\linewidth]{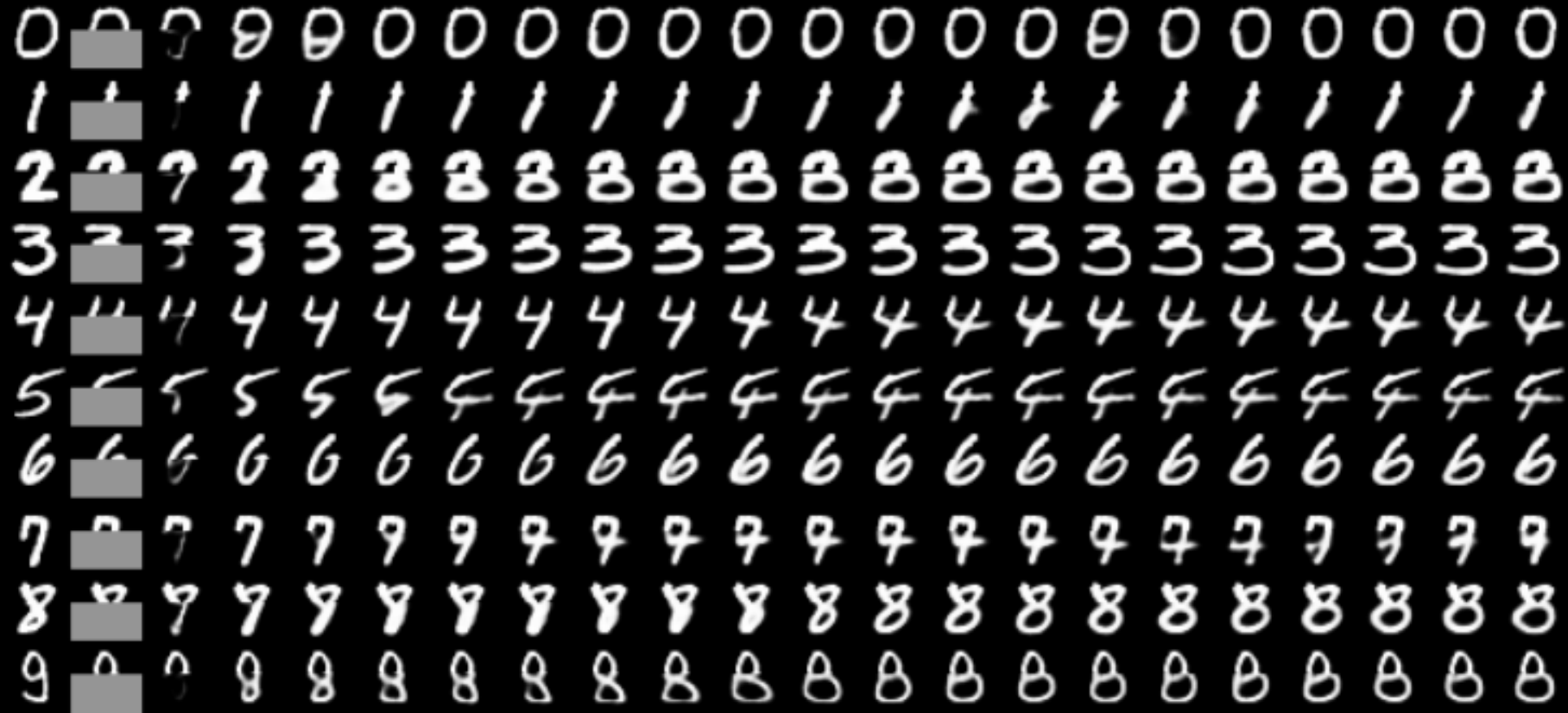}}
  \subfigure[CNN-G + AMC, $T=50$, $\eta = 0.02$]{
 \includegraphics[width=0.65\linewidth]{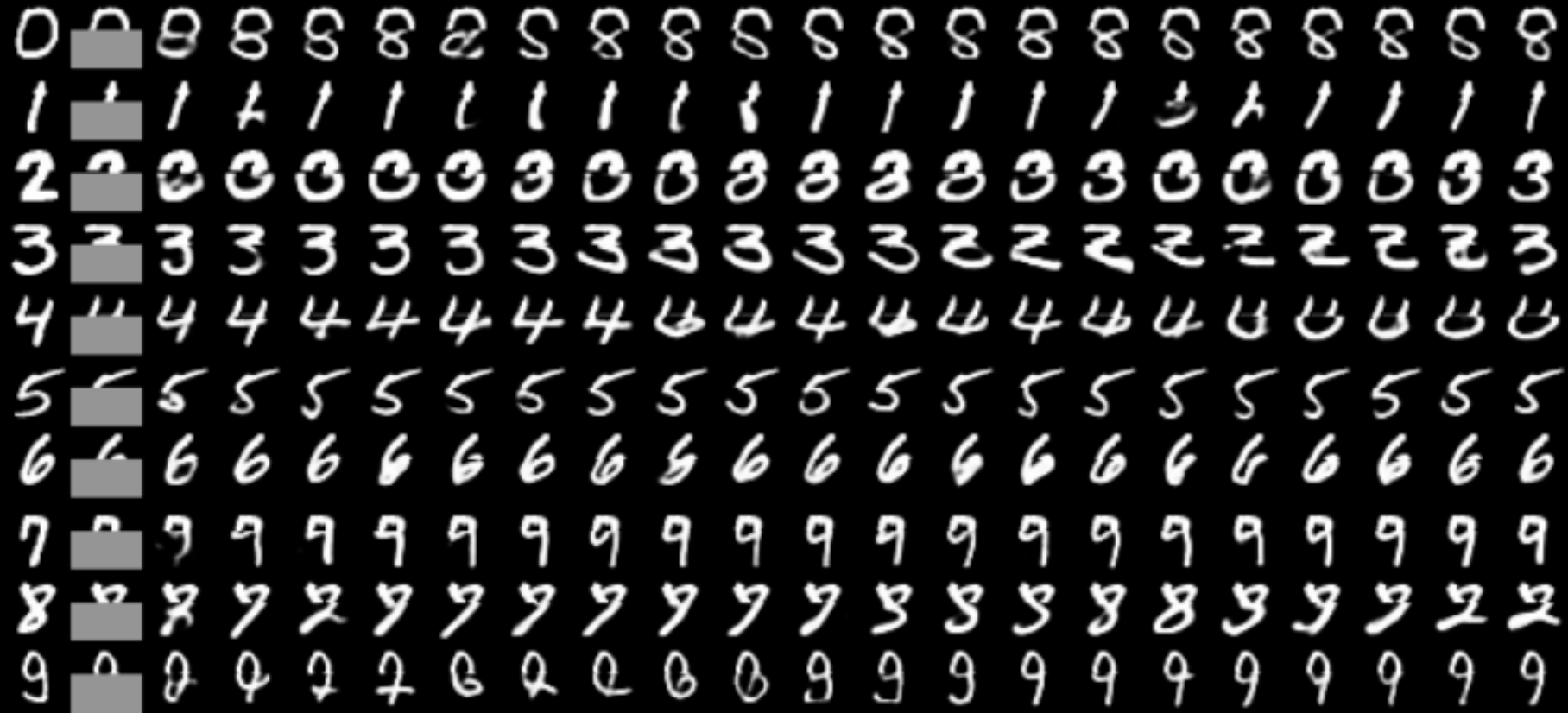}}
 \subfigure[CNN-B + AMC, $T=50$, $\eta = 0.02$]{
 \includegraphics[width=0.65\linewidth]{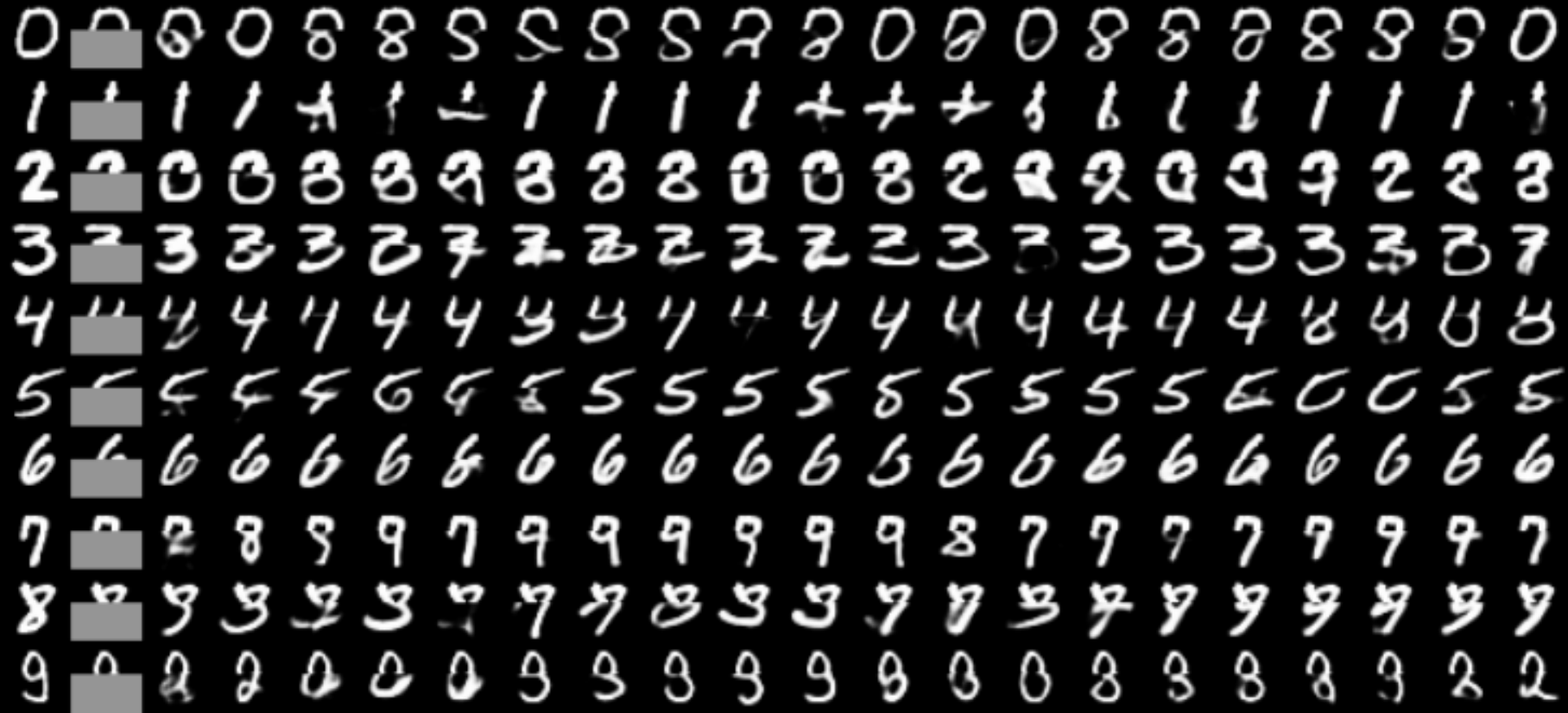}}
 \caption{Missing data imputation results. Removing the lower half pixels.}
\end{figure}
\begin{figure}[t]
 \centering
 \subfigure[Gaussian encoder + VAE]{
 \includegraphics[width=0.65\linewidth]{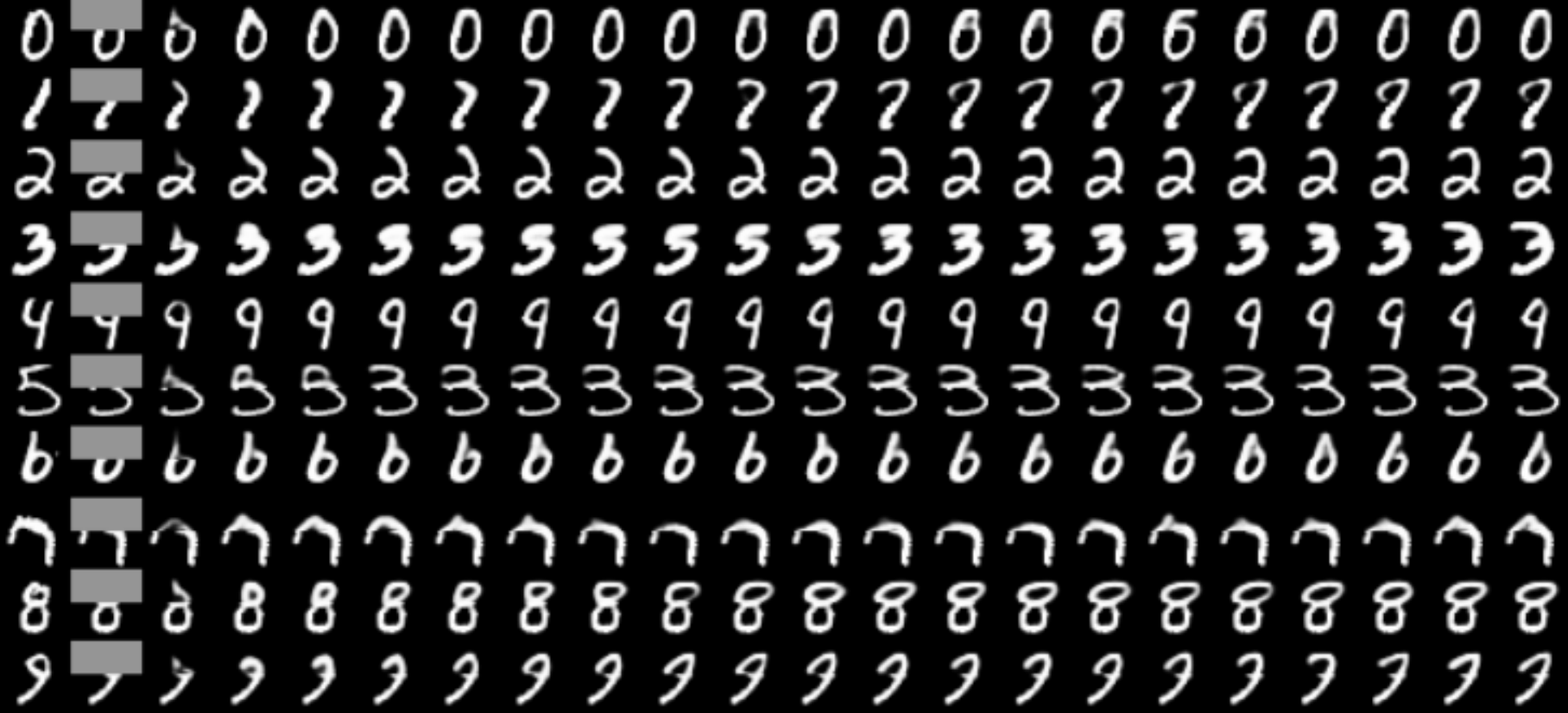}}
  \subfigure[CNN-G + AMC, $T=50$, $\eta = 0.02$]{
 \includegraphics[width=0.65\linewidth]{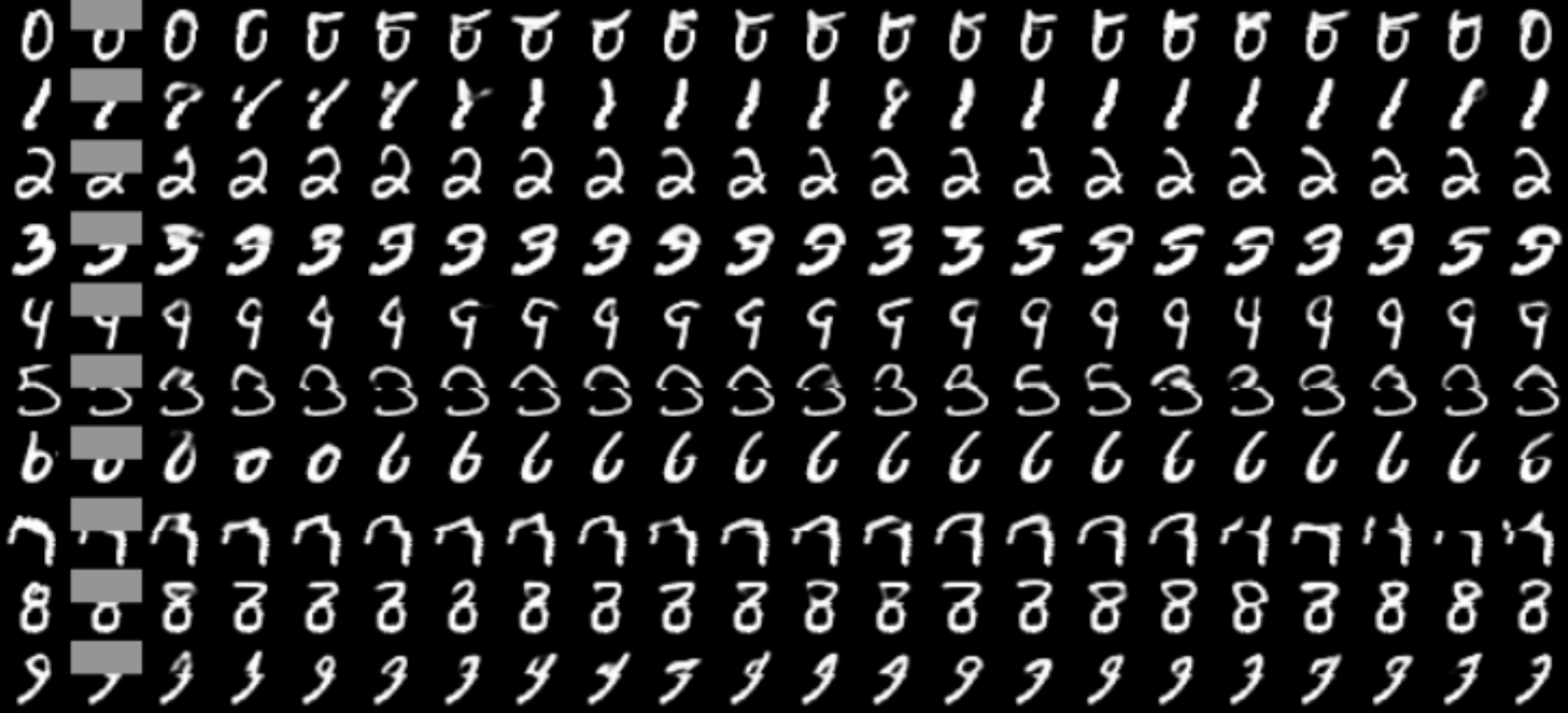}}
 \subfigure[CNN-B + AMC, $T=50$, $\eta = 0.02$]{
 \includegraphics[width=0.65\linewidth]{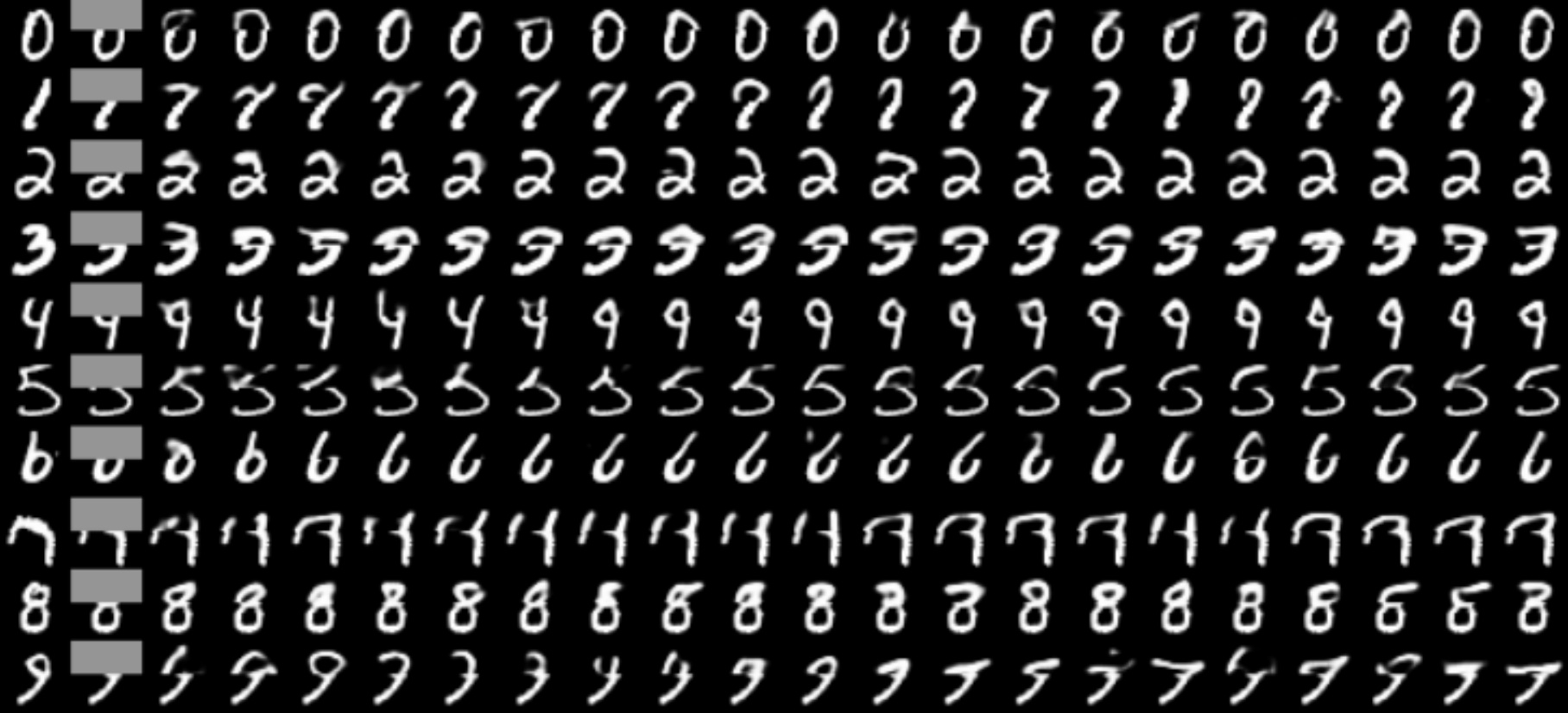}}
 \caption{Missing data imputation results. Removing the upper half pixels.}
\end{figure}
\begin{figure}[t]
 \centering
 \subfigure[Gaussian encoder + VAE]{
 \includegraphics[width=0.65\linewidth]{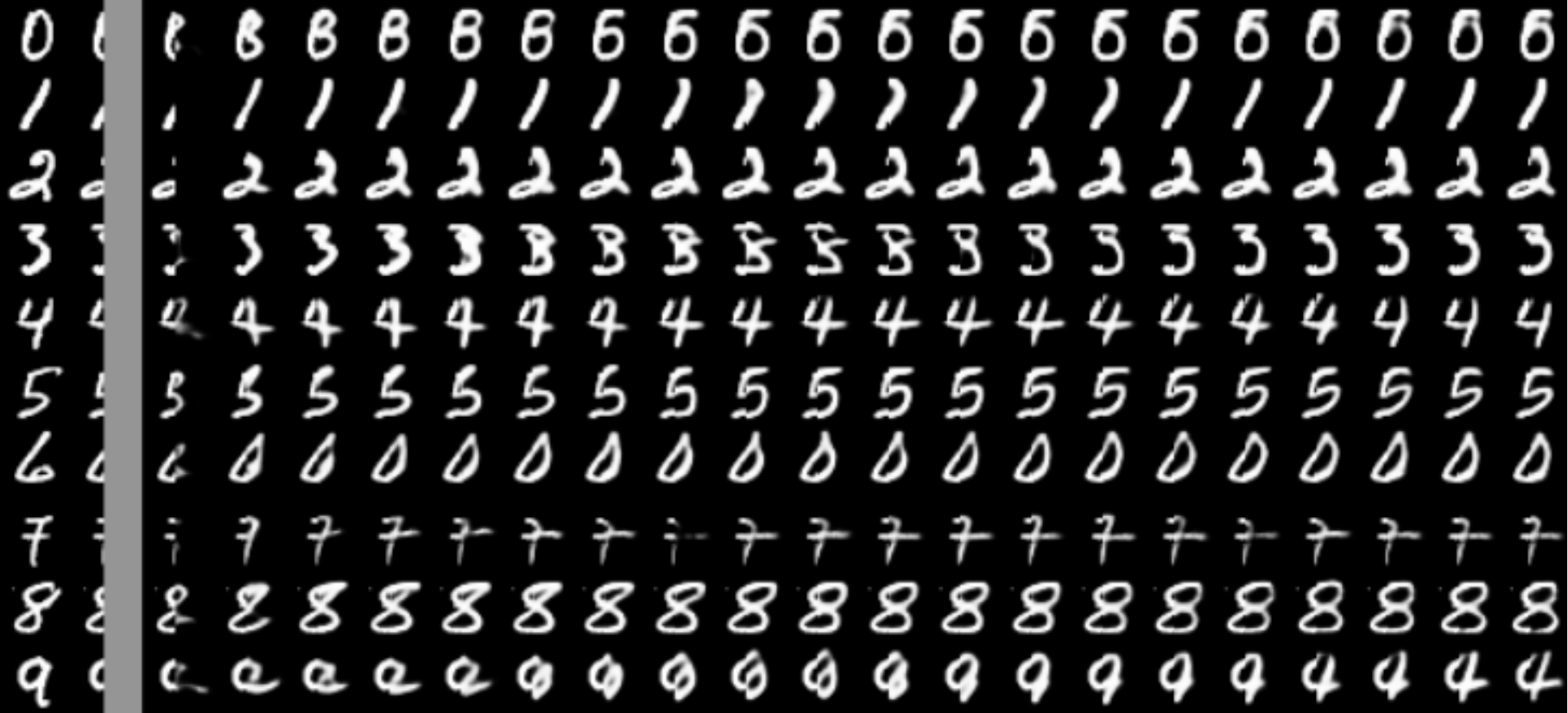}}
  \subfigure[CNN-G + AMC, $T=50$, $\eta = 0.02$]{
 \includegraphics[width=0.65\linewidth]{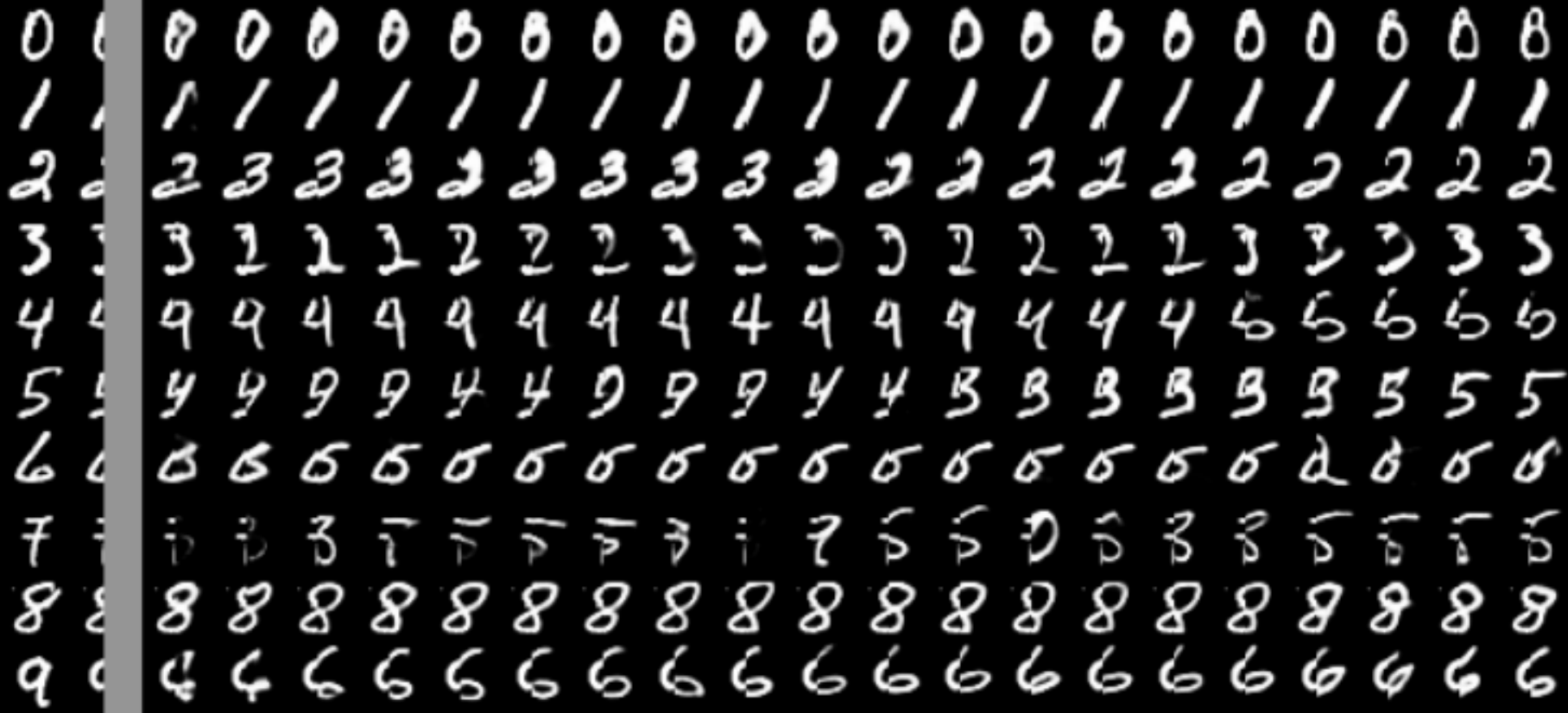}}
 \subfigure[CNN-B + AMC, $T=50$, $\eta = 0.02$]{
 \includegraphics[width=0.65\linewidth]{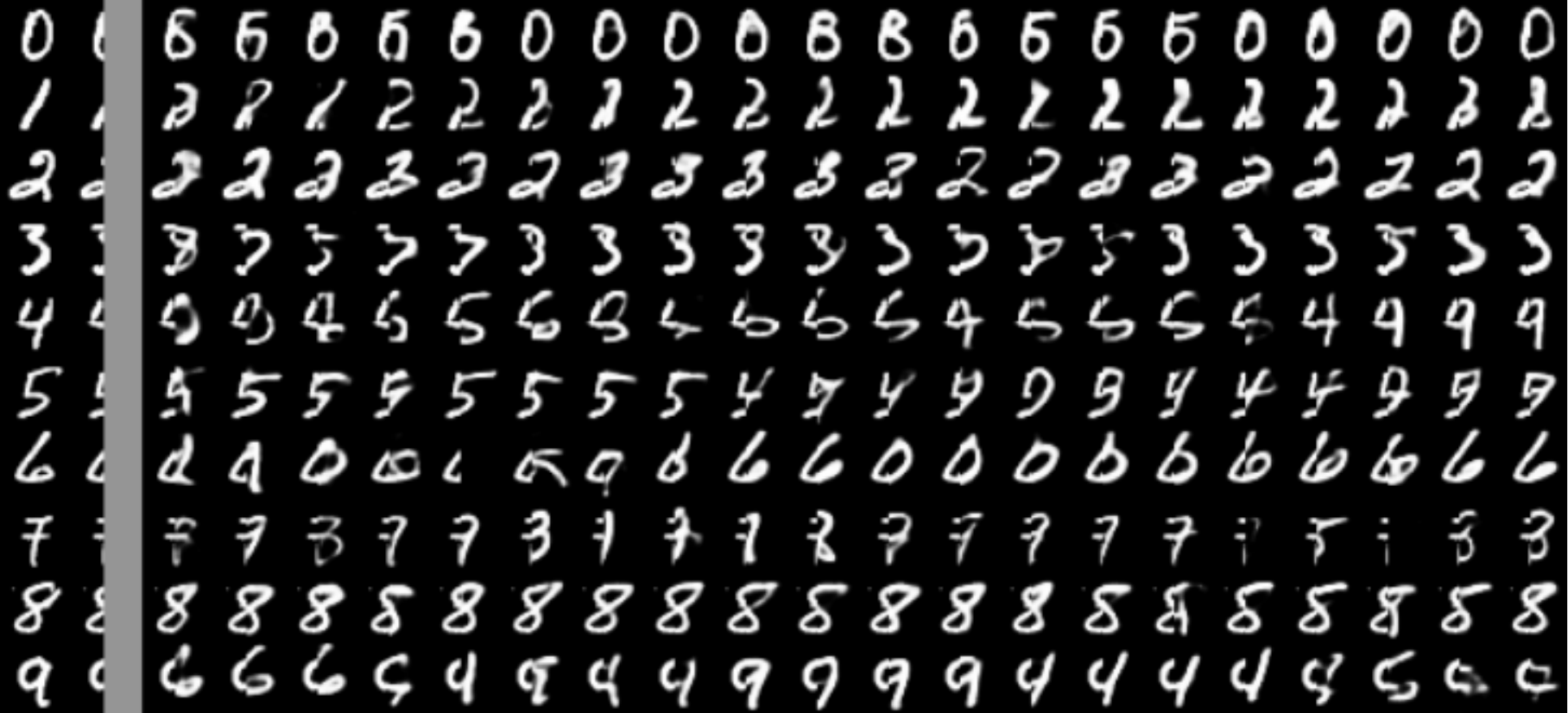}}
 \caption{Missing data imputation results. Removing the right half pixels.}
\end{figure}
\begin{figure}[t]
 \centering
 \subfigure[Gaussian encoder + VAE]{
 \includegraphics[width=0.65\linewidth]{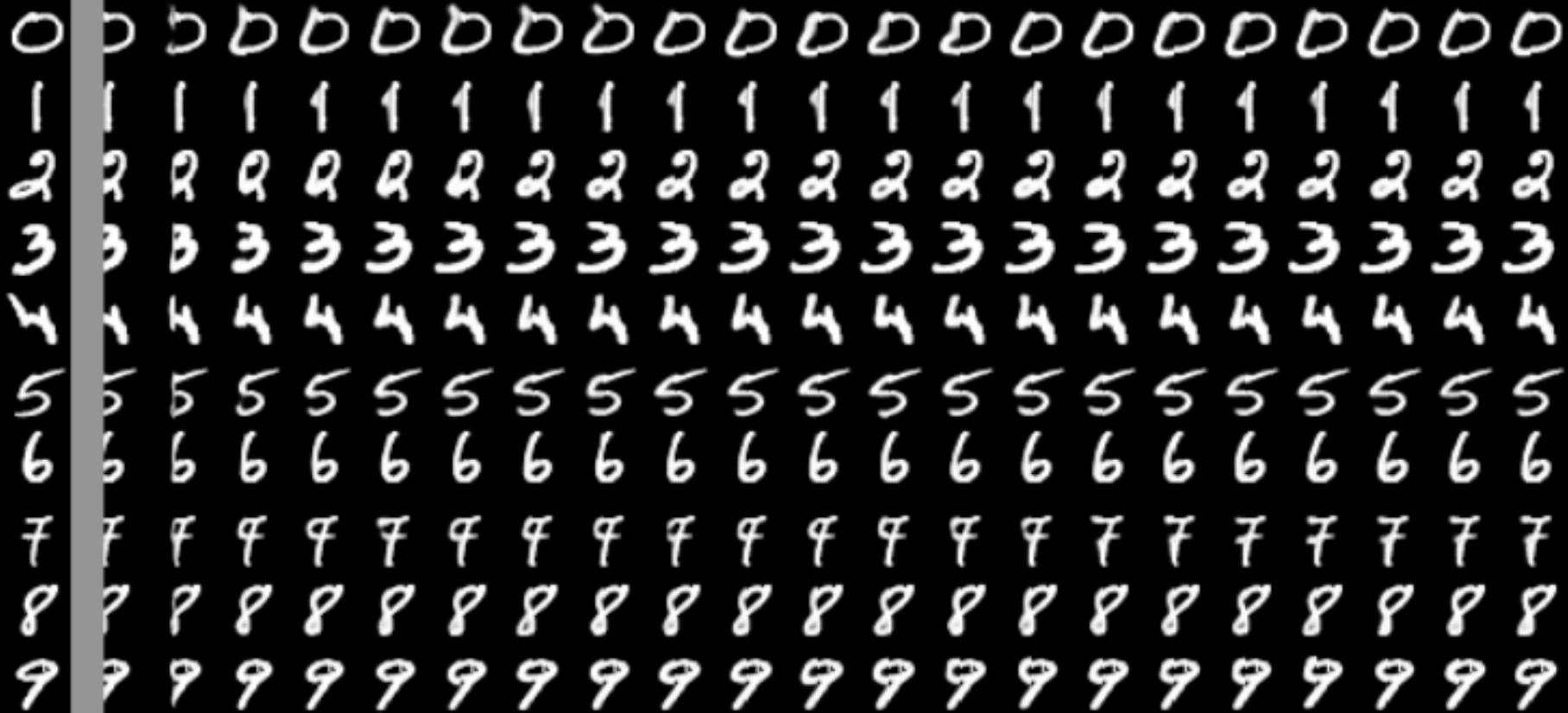}}
  \subfigure[CNN-G + AMC, $T=50$, $\eta = 0.02$]{
 \includegraphics[width=0.65\linewidth]{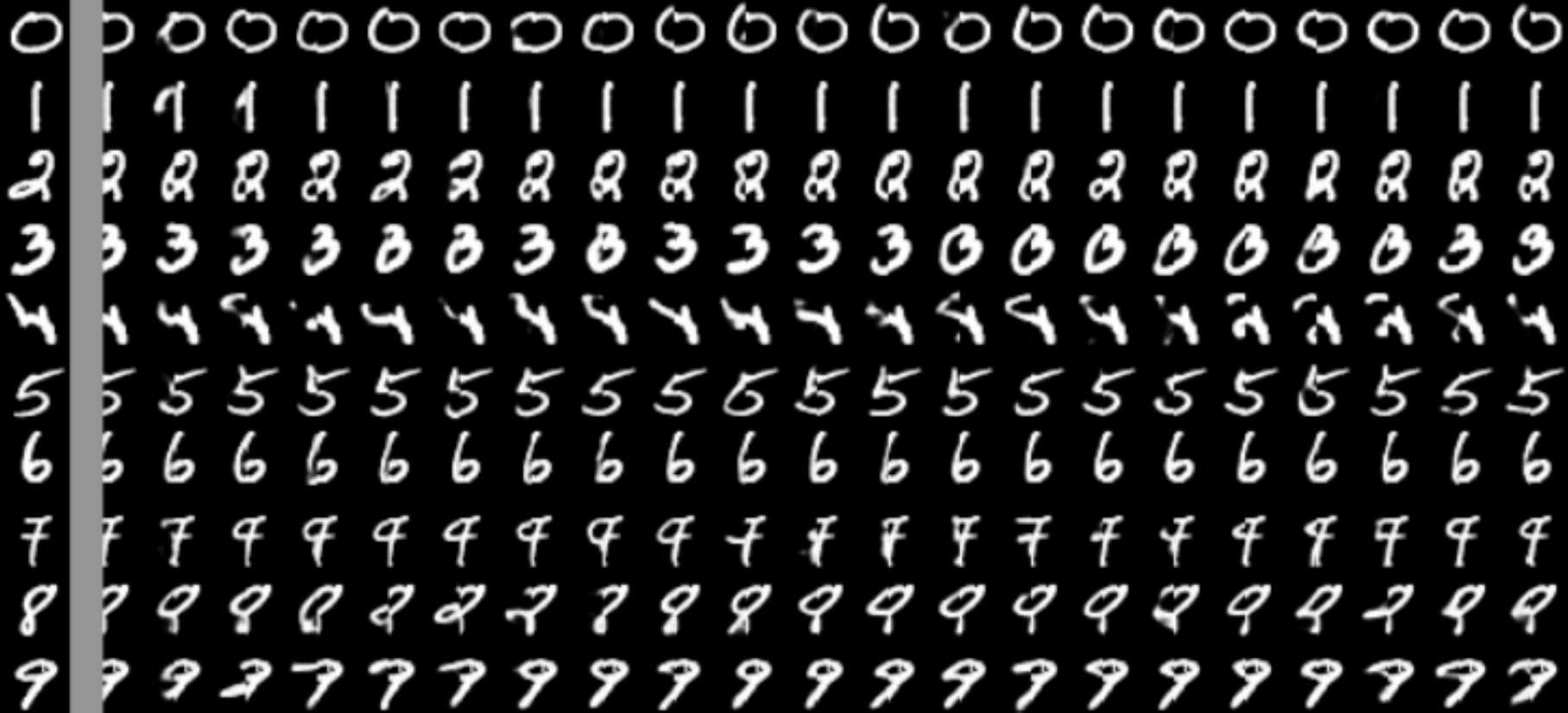}}
 \subfigure[CNN-B + AMC, $T=50$, $\eta = 0.02$]{
 \includegraphics[width=0.65\linewidth]{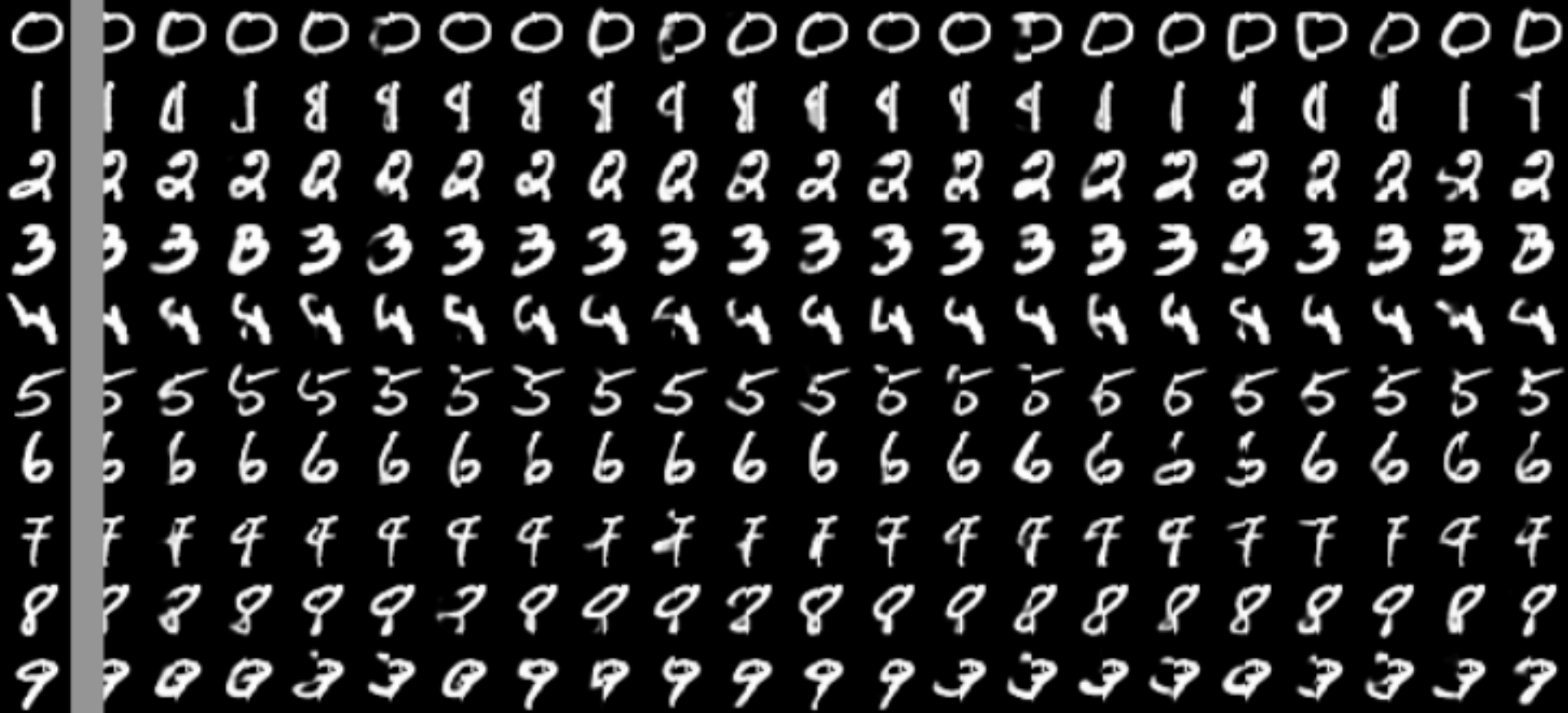}}
 \caption{Missing data imputation results. Removing the left half pixels.}
\end{figure}

\end{document}